\documentclass{article} 
\usepackage{iclr2026_conference,times}

\usepackage{amsmath,amsfonts,bm}

\def\eqref#1{equation~\ref{#1}}

\def\1{\bm{1}}

\DeclareMathAlphabet{\mathsfit}{\encodingdefault}{\sfdefault}{m}{sl}
\SetMathAlphabet{\mathsfit}{bold}{\encodingdefault}{\sfdefault}{bx}{n}

\usepackage{MnSymbol}
\usepackage{amsthm}
\usepackage{amsmath}
\usepackage{hyperref}
\usepackage{url}
\usepackage{graphicx}
\usepackage{subcaption}
\usepackage{multirow}
\usepackage{booktabs}
\usepackage{xspace}
\usepackage{algorithm}
\usepackage{algpseudocode}
\usepackage[table]{xcolor} 
\usepackage[most]{tcolorbox}
\usepackage{soul}
\usepackage{threeparttable}

\newtheorem{proposition}{Proposition}
\newtheorem{theorem}{Theorem}
\newtheorem{lemma}{Lemma}

\renewcommand{\d}{\mathop{}\!\mathrm{d}}
\newcommand{\given}{\,|\,}

\newcommand{\hide}[1]{}

\newcommand{\name}{\textsc{FaReG}\xspace}

\usepackage{colortbl}
\usepackage[most]{tcolorbox}
\newtcolorbox[auto counter, number freestyle={\noexpand\arabic{\tcbcounter}}]{promptbox}[2][]{%
    enhanced,
    colback=blue!5!white,
    colframe=black!75!white,
    title=Prompt~\thetcbcounter: #2,
    #1
}

\title{Fair Conformal Classification via Learning Representation-Based Groups}

\author{Senrong Xu$^1$,~
Yanke Zhou$^1$,~
Yuhao Tan$^1$,~
Zenan Li$^2$,~
Yuan Yao$^1$,~
Taolue Chen$^3$,\\
\textbf{Feng Xu}$^1$,~ 
\textbf{Xiaoxing Ma}$^1$\\
\\
$^1$ State Key Lab of Novel Software Technology, Nanjing University, China, \\
$^2$ ETH Zürich, $^3$ Birkbeck, University of London \\
\texttt{\{srxu,yankezhou,yhtan\}@smail.nju.edu.cn}, \texttt{zenan.li@inf.ethz.ch} \\
\texttt{\{y.yao,xf,xxm\}@nju.edu.cn}, \texttt{t.chen@bbk.ac.uk}
}

\iclrfinalcopy 
\begin{document}

\maketitle

\begin{abstract}
Conformal prediction methods provide statistically rigorous marginal coverage guarantees for machine learning models, but such guarantees fail to account for algorithmic biases, thereby undermining fairness and trust. 
This paper introduces a fair conformal inference framework for classification tasks. 
The proposed method constructs prediction sets that guarantee conditional coverage on adaptively identified subgroups, which can be implicitly defined through nonlinear feature combinations.
By balancing effectiveness and efficiency in producing compact, informative prediction sets and ensuring adaptive equalized coverage across unfairly treated subgroups, our approach paves a practical pathway toward trustworthy machine learning. Extensive experiments on both synthetic and real-world datasets demonstrate the effectiveness of the framework.
\end{abstract}


\section{Introduction}\label{sec:intro}
The rapid advancement of modern machine learning models, especially deep neural networks, has enabled their deployment in high-stake decision-making situations such as medical diagnoses~\citep{kaur2020medical}, resume filtering~\citep{deshpande2020mitigating}, and financial fraud detection~\citep{kamuangu2024review}. Despite their strong average performance, real-world deployment raises critical challenges, notably in uncertainty quantification~\citep{guo2017calibration,ahmed2023deep} and algorithmic fairness~\citep{berk2024improving,almasoud2025algorithmic}.

Ensuring reliable decision-making necessitates the development of unbiased uncertainty measures, as even highly accurate models are prone to producing over-confident and erroneous predictions~\citep{ovadia2019can}.
Conformal Prediction (CP,~\citep{vovk2005algorithmic,smith2024uncertainty}) has emerged as a key framework for providing distribution-free, model-agnostic prediction sets with user-specified (marginal) coverage guarantees. These sets provide reliable uncertainty information for decision-makers especially when the set size is small (i.e., with high efficiency).

On the other hand, algorithmic biases often manifest as disproportionately poor performance on the subgroup defined by specific feature conditions (e.g., \textit{Race=Black \& Gender=Female}), which may arise from imbalanced data distribution or model inherent limitations~\citep{hellman2020measuring}. These biases underscore the need for algorithmic fairness mechanisms that extend beyond average performance to ensure equitable treatment across all groups~\citep{fabris2022algorithmic,das2023algorithmic}. 
However, there may exist tensions between the efficiency of CP and algorithmic fairness, because the former desires a small prediction set, while the latter 
may necessitate 
larger sets for equal conditional coverage across all subgroups~\citep{gibbs2025conformal}.

Conformal prediction with \emph{equalized coverage}~\citep{romano2020malice} provides a pragmatic approach to the efficiency–fairness trade-off. This approach ensures that the target coverage level (e.g., 90\%) is satisfied not only marginally over the entire population, but also conditionally on each protected group of interest. 
However, acquiring prediction sets with equalized coverage is challenging,
as the number of all plausible groups of interest is exponential in the number of features.
A straightforward enumeration is practically infeasible both statistically and computationally, 
especially on multi-dimensional (continuous) features.
Indeed, \citet{romano2020malice} only takes each single feature as the condition of groups (e.g., a group defined by $\textit{Gender=Female}$), which is an arguably insufficient 
representation of the entire space of groups.




\begin{figure}[t]
    \centering
    \includegraphics[width=0.6\linewidth]{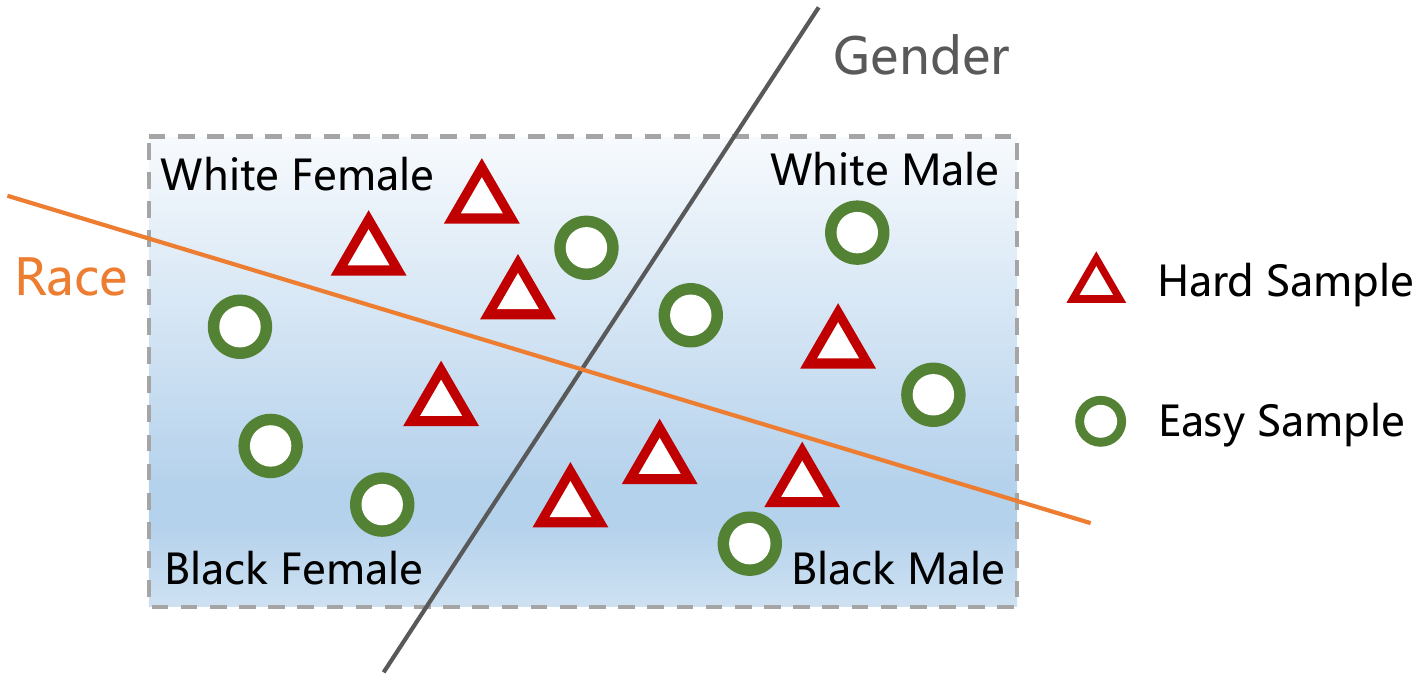}
    \vskip -0.5em
    \caption{An illustrative example. The group space is divided into four parts by the feature {\em Race} and {\em Gender}. Hard samples (red triangles) are unfairly treated by the classifier, and easy samples (green circles) are normally treated. Note that a single feature (either Race or Gender) cannot discover the unfair subgroups (both have four triangles and four circles). Stronger expressiveness is desirable to capture the unfair subgroup ``\textit{White Female} or \textit{Black Male}''.}
    \label{F:intro}
\end{figure}


Later, \citet{zhou2024conformal} observe that 
algorithmic biases often concentrate on a minority of subgroups, 
and propose adaptively fair conformal prediction (AFCP) 
to identify these potentially disadvantaged subgroups. 
In a nutshell, AFCP computes the conditional coverage score for each discrete feature and selects the top-$k$ sensitive features with a greedy strategy (where $k$ is a hyperparameter). 
However, this group identification method still has limited expressiveness. For example, it cannot capture 
groups defined by a nonlinear combination of features, such as Exclusive OR (see the subgroup ``\textit{White Female} or \textit{Black Male}'' 
in Fig.~\ref{F:intro}).
Additionally, AFCP is based on Na\"{i}ve Bayes, which incurs a high computational cost and restricts its applicability to continuous features.
 

\subsection{Our Contributions}
In this paper, we propose a new group-fair conformal prediction method, \textbf{fair conformal prediction for representation-based groups} (\name), which accommodates both group expressiveness and time efficiency.
Different from existing work~\citep{romano2020malice,zhou2024conformal} which directly extracts groups from the raw input feature $X$, our approach encodes $X$ into a 
latent representation $Z$ via a mapping $Z=f(X)$, and learns unfair groups characterized by the low group coverage 
based on $Z$.
The introduction of 
$Z$ as a high-level representation of features strengthens the expressiveness of models,  
allowing a thorough exploration of groups.  
Meanwhile, we can enhance the interpretability by reconstructing input $X$ from the encoding $Z$. 
To this end, we carefully design an encoder-decoder architecture and the optimization objective, based on the principle of variational inference.

In addition, we propose a \emph{nonlinear} version of the conditional coverage metric WSC~\citep{cauchois2021knowing}, namely $\text{WSC}^+$, aiming to evaluate the conditional coverage of unfairly treated groups more precisely. This allows users to check a conformal procedure and to compare multiple alternative conformal procedures. 


The main contributions of this paper are summarized as follows. 
First, we propose a new conformal prediction method 
to enhance the expressiveness of unfair group identification.
Second, we extend the traditional conditional coverage metric WSC to a nonlinear version $\text{WSC}^+$ for more accurate evaluation.
Comprehensive experiments on both synthetic and real-world datasets confirm the effectiveness and efficiency of our proposed method.



\section{Preliminary} \label{sec:prel}

For any natural number $n$, we write $[n]:=\{1, \ldots, n\}$.
We work with the most widely-used version of conformal prediction, i.e., \emph{split} conformal prediction, where we assume a calibration set $\mathcal{D}=\{(X_i,Y_i)\}_{i=1}^N$ of i.i.d. (or simply exchangeable) observations sampled from an (unknown) distribution $P_{XY}$. In standard classification, $X_i\in \mathcal{X}$ represents the input feature from a feature space $\mathcal{X}\subseteq\mathbb{R}^d$ and $Y_i \in [L]$ is a categorical label. 
A given classifier $\hat{f}$ is trained (on a training set) to predict the conditional distribution $P(Y \given X)$. Furthermore, $X_{N+1}$ is a test instance with an unknown label $Y_{N+1}$ sampled by $P_{XY}$.
%
CP constructs a prediction set $C(X_{N+1})$ for  $Y_{N+1}$ based on $\mathcal{D}$. 
The output $C(X_{N+1})$ guarantees marginal coverage at a user-specified level, i.e.,
\[
\mathbb{P}[Y_{N+1} \in C(X_{N+1})] \geq 1 - \alpha,
\]
where $\alpha\in (0,1)$ is a predefined miscoverage rate. 

Typically, CP proceeds in three steps: 
(1) computing the predefined conformity score $V(x_i,y_i)$ for each sample $(x_i,y_i)\in\mathcal{D}$ using the predictive results of the classifier $\hat{f}$; 
(2) setting $(1-\alpha)(1+1/N)$-quantile score of $\mathcal{D}$ as a threshold $\hat{\eta}$; 
(3) constructing the prediction set $C_\frak{m}(X_{N+1},\mathcal{D}):=\{y\in [K] \given V(X_i,y)\geq \hat{\eta}\}$, which is used as $C(X_{N+1})$ for $X_{N+1}$. 

It can be shown that $C_\frak{m}(X_{N+1},\mathcal{D})$ meets the desirable \emph{marginal coverage}. 
Intuitively, marginal coverage implies that the prediction set is guaranteed to contain the true label with the \emph{average} $1-\alpha$ probability over the population. 
However, this guarantee is deemed to be insufficient, 
especially when miscoverage exhibits systematic bias, disproportionately affecting individuals belonging to groups characterized by certain features.  

By contrast, \emph{conditional coverage} requires $\mathbb{P}[Y_{N+1} \in C(X_{N+1}) \given X_{N+1}=x] \geq 1 - \alpha$ for each $x\in \mathcal{X}$.  
This is much stronger 
as it demands correct coverage across all regions of the feature space, not just on average. However, achieving conditional coverage is impossible 
without imposing extra assumptions on the underlying distribution $P_{XY}$ (such as 
the smoothness of $P_{XY}$~\citep{cai2014adaptive,lei2014distribution} and strictly limiting the size of feature space $\mathcal{X}$~\citep{lee2021distribution}). As these strong assumptions are often violated, conditional coverage is less meaningful in practice.
%
 
Equalized coverage~\citep{romano2020malice} represents a pragmatic compromise   
to ensure validity across \emph{predefined} sample groups that need to be protected. 
Given a group $\mathcal{G}\subseteq \mathcal{X}$, 
it is required that 
\begin{equation*}
\mathbb{P}[Y_{N+1} \in C(X_{N+1}) \given X_{N+1}\in \mathcal{G}] \geq 1 - \alpha
\end{equation*}
for all $\mathcal{G}$ of interest. 
In particular, these groups are typically related to some specific features called sensitive features.

However, the requirement for rigorous equalized coverage is localized, as algorithmic biases disproportionately affect only a minority of subgroups~\citep{zhou2024conformal}, as mentioned in Seciton~\ref{sec:intro}.
Therefore, AFCP further proposes adaptive equalized coverage based on equalized coverage, formalized by
\begin{equation}\label{Eq:Adaptive equalized coverage}
\mathbb{P}[Y_{N+1} \in C(X_{N+1}) \given X_{N+1}\in \hat{\mathcal{G}}] \geq 1 - \alpha,
\end{equation}
where $\hat{\mathcal{G}}$ is adaptively selected corresponding to sensitive features.
Eq.~\ref{Eq:Adaptive equalized coverage} indicates that $C(X_{N+1})$ is well-calibrated for the selected group $\hat{\mathcal{G}}$ defined by these sensitive features.

%




\section{Methodology} \label{sec:method}

This section presents \name, a learning-based method that adaptively identifies groups affected by algorithmic bias and adjusts their prediction sets to achieve equalized coverage while preserving high informativeness.

\subsection{Learning Representation-Based Groups}\label{sec:learning groups}
\noindent\textbf{Optimization Objective.}
%
For any feature $x\in \mathcal{X}$, we write its encoding $z=f(x)\in \mathcal{Z}$, where $\mathcal{Z}$ is a latent representation space.  
Intuitively, $z$ denotes the latent representations of feature combinations of $x$. 
We introduce a random 
binary variable $S$ and $Z$ taking values in $\mathcal{Z}$ to formalize the membership of a group. Naturally, we consider a conditional distribution $P(S \given Z)$ 
such that the probability of $x\in \hat{\mathcal{G}}$ for a group $\hat{\mathcal{G}}$ is equal to $\mathbb{P}(S=1\given Z=f(x))$.
Our goal is twofold: (1) to learn an encoding $Z=f(X)$ that is maximally informative about $S$ and $X$, while (2) 
$Z$ 
does not reveal the identity of any individual $i$ in the sample (e.g., the calibration set). 

We apply the \emph{deep variational information bottleneck} (Deep VIB) method~\citep{alemi2016variational}. 
Specifically, for two random variables $X$ and $Y$ with the joint pdf (parameterized by $\theta$), $p_\theta(x,y)$, $I(X,Y;\theta)=\int p_\theta(x,y)\log\frac{p_\theta(x,y)}{p_\theta(x)p_\theta(y)}~dxdy$ denotes their mutual information. The optimization objective can be formalized as
\begin{equation*}
\max I(Z,S;\theta) + I(Z,X;\theta) - \beta I(Z,i;\theta),
\end{equation*}
where $i$ is a random variable to take any instance from the sample (e.g., in this paper, the calibration set $\mathcal{D}$) with a uniform distribution, $\theta$ is the model parameter, and $\beta$ is a weight hyperparameter. (We abbreviate $I(Z,S),I(Z,X),I(Z,i)$ as $I_1,I_2,I_3$ for convenience.)

By introducing $q_\phi(s|z),q_\varphi(x|z),r(z)$ as the variational approximation to $p_\theta(s|z),p_\theta(x|z),p(z)$ in respective terms, we perform variational inference 
and obtain
\begin{equation*}
\begin{aligned}
I_1 + I_2 - \beta I_3 &\geq \int p_\theta(x)p_\theta(s \given x)p_\theta(z \given x)\log q_\phi(s \given z) \d x \d s \d z \\
&\quad + \int p_\theta(x)p_\theta(z \given x)\log q_\varphi(x \given z) \d x \d z - \frac{\beta}{N}\sum_i\int p_\theta(z \given x_i)\log \frac{p_\theta(z \given x_i)}{r(z)}\d z.
\end{aligned}
\end{equation*}
(The details 
are given in Appendix~\ref{appendix:var}.)

In practice, we can approximate $p_\theta(x,s)=p_\theta(x)p_\theta(s|x)$ and $p_\theta(x)$ using the empirical distribution on the observations (e.g., the calibration set $\mathcal{D}$). 
As for $p_\theta(z|x)$, the reparameterization trick~\citep{kingma2013auto} forces $z$ to conform to a normal distribution which relies on $x_i$, and hence its deterministic function 
can be rewritten as $z=f(x,\epsilon)$ with an (auxiliary) noise variable $\epsilon$. 

Substituting all of these into the above equation, we obtain the following loss function 
\begin{equation}\label{Eq:Loss Funciton_1}
\mathcal{L} = -\frac{1}{N}\sum_{i=1}^N \left( \mathbb{E}_{\tilde{z}\sim f(x_i,\epsilon)}[\log q_\phi(s_i\given \tilde{z}) + \log q_\varphi(x_i\given \tilde{z})] - \beta D_{\mathrm{KL}}(p_\theta(z \given x_i)\Vert r(z))\right).
\end{equation}

Intuitively, the expected log-likelihood $\mathbb{E}_{\tilde{z}\sim f(x_i,\epsilon)}[\log q_\phi(s_i\given \tilde{z}) + \log q_\varphi(x_i\given \tilde{z})]$ allows the encoding $\tilde{z}$ to predict $s_i$ and regenerate $x_i$ simultaneously, whereas the Kullback-Leibler (KL) divergence aims to compress the remaining useless information of $\tilde{z}$.


\noindent\textbf{Instantiation.}
Eq.~\ref{Eq:Loss Funciton_1} 
suggests a natural design of the Encoder-Decoder architecture. 
In our method, the stochastic encoder with parameter $\theta$ has the form $p_\theta(z\given x)=\mathcal{N}(z\given f_\mu(x), f_\sigma(x))$, where $f_\mu(x)$ and $f_\Sigma(x)$ are two MLP networks to output the mean and variance of a normal distribution. We set $r(z)$ as a standard normal distribution $\mathcal{N}(0,1)$ and directly minimize the KL divergence term in Eq.~\ref{Eq:Loss Funciton_1} using the reparameterization trick.


We now concentrate on two decoders with parameters $\phi$ and $\varphi$. 
The instantiation of decoder with parameter $\varphi$ is trivial. For the expected log-likelihood $\mathbb{E}_{\tilde{z}\sim f(x_i,\epsilon)}[\log q_\varphi(x_i\given \tilde{z})]$ in Eq.~\ref{Eq:Loss Funciton_1}, we utilize the standard Mean Squared Error (MSE) as the reconstruction loss~\citep{kingma2013auto}.

Decoder with parameter $\phi$ aims at predicting $S$, 
which indicates whether the sample $X$ belongs to group $\hat{\mathcal{G}}$ or not.
Assume 
a set of observations, e.g., 
the calibration set $\mathcal{D}=\{(X_i,Y_i)\}_{i=1}^N$. 
The distribution $P(S\given X)$ 
can be viewed as a binary classifier $h$ comprising an encoder with parameter $\theta$ and a decoder with the parameter $\phi$. 
The result of $h$ on $\mathcal{D}$ is a vector $\mathbf{s}=[s_1,\dots,s_N]\in \{0,1\}^N$. 
Let $\hat{\mathcal{G}}_\mathbf{s}\subseteq \mathcal{D}$ 
denote the group determined by $\mathbf{s}$ on $\mathcal{D}$ and 
$\mathcal{H}$ be the family of all plausible $h$.
We extend an inequality~
\citep{cauchois2021knowing} 
to measure the deviation between the empirical coverage probability $\mathbb{P}_n$ on $\mathcal{D}$ and the oracle coverage probability $\mathbb{P}$.

\begin{proposition} \label{prop:finite}
Let the VC-dimension $VC(\mathcal{H})\leq R$ and $\delta=\lvert \hat{\mathcal{G}}_\mathbf{s} \rvert / N$ be the proportion of $\hat{\mathcal{G}}_\mathbf{s}$ to the entire dataset. Then the gap between the empirical coverage probability $\mathbb{P}_n$ on the observations and the oracle coverage probability $\mathbb{P}$ is upper bounded, i.e., there exists some constant $C_1$ for all $\tau > 0$ 
\begin{equation*}
\sup_{h\in\mathcal{H}}\left\{\lvert \mathbb{P}_n[Y \in C(X) \given X \in \hat{\mathcal{G}}_\mathbf{s}] - \mathbb{P}[Y \in C(X) \given X \in \hat{\mathcal{G}}_\mathbf{s}] \rvert \right\}  \leq C_1\sqrt{\frac{R\log N + \tau}{\delta N}}
\end{equation*} 
holds with probability at least $1-e^{-\tau}$.
\end{proposition}

Proposition~\ref{prop:finite} (cf.~Appendix~\ref{A:prop1} for proof) highlights two key directions for reducing the discrepancy between $\mathbb{P}_n$ and $\mathbb{P}$. First, a lower VC-dimension $VC(h)$ leads to a more precise estimation $\mathbb{P}_n$, implying that the classifier $h$ should exhibit limited complexity. Second, the selected group must be sufficiently large to ensure reliable estimation.

We maximize the expected log-likelihood $\mathbb{E}_{\tilde{z}\sim f(x_i,\epsilon)}[\log q_\phi(s_i \given \tilde{z})]$ in Eq.~\ref{Eq:Loss Funciton_1} via minimizing the expected empirical conditional coverage of the selected group $\hat{\mathcal{G}}$. 
The group $\hat{\mathcal{G}}$ on 
$\mathcal{D}$ is determined by a random vector $\mathbf{S}$, sampled from a joint Bernoulli distribution $B= \prod_{i=1}^N \operatorname{Bernoulli}(q_\phi(S_i=1 \given \tilde{z}))$.  
Hence, given $\mathcal{D}$, we formulate the following optimization problem :
\begin{equation}\label{Eq:minimize coverage}
\min_\phi \mathbb{E}_{\mathbf{S}\sim B}[\mathbb{P}_n[Y \in C(X) \given X \in \hat{\mathcal{G}}_\mathbf{S}]] \quad \text{s.t.} ~~\frac{1}{N}\sum_{i=1}^Nq_\phi(S_i=1 \given \tilde{z}) \geq \delta.
\end{equation}
In the above minimization problem, $\delta=\lvert \hat{\mathcal{G}}_\mathbf{s} \rvert/N$ denotes the the proportion of the selected group size to the whole dataset $\mathcal{D}$, and the decoder with parameter $\phi$ is a simple logistic regression model of the form $q_\phi(s \given \tilde{z})=\sigma (s\given f_m(\tilde{z}))$, where $\sigma$ is the sigmoid function and $f_m$ is a MLP network.

To solve the constrained optimization problem, we employ the Projected Gradient Descent (PGD), an iterative optimization algorithm~\citep{madry2017towards}, to optimize the parameter $\phi$.
In each training step, PGD performs a gradient descent update and then projects the new point onto the feasible set to ensure all constraints are satisfied.
Specifically, when the predictive distribution $q_\phi(s\given \tilde{z})$ does not meet the constraint $ \frac{1}{N}\sum_{i=1}^Nq_\phi(S_i=1\given \tilde{z}) \geq \delta$ after one back propagation process, we project it back onto the constraint-friendly space. 
Such a projection is equivalent to an $\ell_2$ distance minimization problem. Let $q_\phi^*(s_{1}\given \tilde{z}) \geq \dots\geq q_\phi^*(s_{N}\given \tilde{z})$ be the descending order of $\{q_\phi(s_i\given \tilde{z})\}_{i=1}^N$, and  the projection results in 
\begin{equation}\label{Eq:project}
q_\phi'(s_i\given \tilde{z}) = \min\left(1,q_\phi(s_i\given \tilde{z})+\frac{\omega}{2}\right),
\end{equation}
where $\omega = (\delta-k-\sum_{i=k+1}^Nq_\phi^*(s_{i}\given \tilde{z}))/ (N-k) \geq 0$, $k\in [N]$ is the greatest index to satisfy $q_\phi^*(s_{k}\given \tilde{z})+\omega/2\geq 1$ and $q_\phi^*(s_{k+1}\given \tilde{z})+\omega/2< 1$.
(The details are given in Appendix~\ref{A:opt}.)

Overall, we employ the empirical conditional coverage loss $\mathcal{L}_{\mathrm{CC}}$, the reconstruction loss $\mathcal{L}_{\mathrm{MSE}}$, and the KL divergence loss $\mathcal{L}_{\mathrm{KL}}$ to replace the corresponding terms in Eq.~\ref{Eq:Loss Funciton_1}, resulting in 
\begin{equation}\label{Eq:Loss Funciton_2}
\mathcal{L} = \mathcal{L}_{\mathrm{CC}} +  \mathcal{L}_{\mathrm{MSE}} - \beta\mathcal{L}_{\mathrm{KL}}.
\end{equation}

\subsection{Constructing the Adaptive Prediction Sets}\label{sec:Constructing}

After selecting the unfair group $\hat{\mathcal{G}}$, we proceed to construct the final prediction set with $\hat{\mathcal{G}}$. First, a standard conformal prediction set $C_\frak{m}(X_{N+1},\mathcal{D})$ is constructed using classic adaptive conformal prediction. Then, we perform $T$ sampling of the vector $\mathbf{s}_t$ $(t\in [T])$ from the joint Bernoulli distribution $B$ learned by models in Eq.~\ref{Eq:minimize coverage}. Each $\mathbf{s}_t$ defines a group $\hat{\mathcal{G}}_{\mathbf{s}_t}$, and such group is used as a calibration set to build a prediction set $C_\frak{m}(X_{N+1},\hat{\mathcal{G}}_{\mathbf{s}_t})$ as mentioned in Section~\ref{sec:prel}. The final prediction set for $Y_{N+1}$ is given by the union of all these sets:
\hide{
After selecting the group $\hat{\mathcal{G}}$, we show how to construct prediction sets with $\hat{\mathcal{G}}$. 
First, we construct $C_\frak{m}(X_{N+1},\mathcal{D})$ using the classic adaptive conformal prediction. 
We then repeat $T$ times to sample the vector $\mathbf{s}_t$ from the joint Bernoulli distribution $B$ in Eq.~\ref{Eq:minimize coverage}, and construct the prediction set $C_\frak{m}(X_{N+1},\hat{\mathcal{G}}_{\mathbf{s}_t})$ based on the selected group $\hat{\mathcal{G}}_{\mathbf{s}_t}$. 
Finally, 
the prediction set for $Y_{N+1}$ is given by
}
\begin{equation}\label{Eq:Construct}
C(X_{N+1})=C_\frak{m}(X_{N+1}, \mathcal{D})\cup \bigcup_{t=1}^T C_\frak{m}(X_{N+1},\hat{\mathcal{G}}_{\mathbf{s}_t}).
\end{equation}
Our approach \name is summarized in Algorithm~\ref{alg:framework}. To analyze its time complexity,
assume we have $M$ test instances and the complexity of conducting classic conformal prediction is $\mathcal{O}(N+M)$. 
Then, training the model to select groups is $\mathcal{O}(EN(\vert \theta \vert + \vert \phi \vert + \vert \varphi \vert))$, where $E$ is the number of epochs.
For all $M$ test instances, the time of selecting groups and constructing prediction sets is $\mathcal{O}(TN+TM)$.
The overall complexity of our \name is $\mathcal{O}(EN(\vert \theta \vert + \vert \phi \vert + \vert \varphi \vert) + T(N+M))$, which is 
$\mathcal{O}(N+M)$, disregarding constant multiplicative factors.
In contrast, the complexity of AFCP is $\mathcal{O}(N\log N+NM)$~\citep{zhou2024conformal}. 

\begin{algorithm} 
\caption{The overall framework of \name.} 
\label{alg:framework}
\begin{algorithmic}[1]
\Require calibration dataset $\mathcal{D}=\{X_i,Y_i\}_{i=1}^N$; test instance with feature $X_{N+1}$; list of $K$ sensitive features; pre-trained classifier $\hat{f}$; fixed rule to compute nonconformity scores; level $\alpha\in (0,1)$; selected group size proportion $\delta$; hyperparameter $\beta$; sampling times $T$;
\Ensure prediction set $C(X_{N+1})$; selected group set $\{\hat{\mathcal{G}}_{\mathbf{s}_t}\}_{t=1}^T$.
\State Construct classic conformal prediction set $C_\frak{m}(X_{N+1},\mathcal{D})$ based on the output of $\hat{f}$;
\For {each batch}
\State Calculate KL divergence loss $\mathcal{L}_{KL}$ with reparameterization trick;
\State Sample $\tilde{z}\sim f(x,\epsilon)$;
\State Calculate conditional coverage loss $\mathcal{L}_{CC}$ and reconstruction loss $\mathcal{L}_{MSE}$ using $\tilde{z}$;
\State Put all losses together in $\mathcal{L}$ as defined in Eq.~\ref{Eq:Loss Funciton_2};
\State Update parameters $\theta,\phi$ and $\varphi$ via the gradient descent of $\mathcal{L}$; 
\If{$\sum_{i=1}^Nq_\phi(S_i=1\given \tilde{z})<\delta\cdot N$}
\State Project each $q_\phi(S_i=1\given \tilde{z})$ to satisfy minimum set constraint using Eq.~\ref{Eq:project};
\EndIf
\EndFor
\For {$t\in [T]$}
\State Sample $\mathbf{s}_t\sim B$;{\Comment{\emph{B is a joint Bernoulli distribution mentioned in Eq.~\ref{Eq:minimize coverage}}}}
\State Construct $C_\frak{m}(X_{N+1},\hat{\mathcal{G}}_{\mathbf{s}_t})$;
\EndFor
\State Construct prediction set $C(X_{N+1})$ following Eq.~\ref{Eq:Construct}.
\end{algorithmic}
\end{algorithm}


The following result, proved in Appendix~\ref{A:thm}, ensures that the prediction set $C(X_{N+1})$ generated by \name achieves adaptive equalized coverage (Eq.~\ref{Eq:Adaptive equalized coverage}) over the selected group set $\{\hat{\mathcal{G}}_{\mathbf{s}_t}\}_{t=1}^T$.

\begin{theorem}\label{thm:equalized coverage}
If $\{(X_i,Y_i)\}^{N+1}_{i=1}$ are exchangeable, the prediction set $C(X_{N+1})$ and the selected group set $\{\hat{\mathcal{G}}_{\mathbf{s}_t}\}_{t=1}^T$ output by Algorithm~\ref{alg:framework} satisfy the adaptive equalized coverage defined in Eq.~\ref{Eq:Adaptive equalized coverage},
and this guarantee still holds when the selected groups are defined by a more complex combination of features (e.g., non-linear) compared to AFCP.
\end{theorem}


\section{Experiments} \label{sec:exp}
\subsection{Experimental Setup}\label{sec:settings}
\noindent\textbf{Baselines.}
We select the classic CP method Marginal~\citep{romano2020classification} for classification, the initial CP method Partial~\citep{romano2020malice} considering equalized coverage, and the state-of-the-art method AFCP~\citep{zhou2024conformal} as our baselines.
The vanilla version of AFCP is designed to pick at most one sensitive feature (referred to as AFCP1). 
We also extend AFCP1 to select two sensitive features (referred to as AFCP2), given unreal, strong 
prior knowledge. Note that in real-world applications, it is typically unknown exactly how many features the unfair group may correspond to.

\noindent\textbf{Evaluation Metrics.}
To evaluate the prediction sets $C(X_{N+1})$ produced by different CP methods, we use the coverage conditional on a specific group (referred to as Group Coverage), Average Coverage (viz., marginal coverage), and Average Size (viz., efficiency) as the metrics.

Additionally, we propose a new conditional coverage metric, viz., $\text{WSC}^+$, to capture groups defined by complicated (nonlinear) feature relationships. 
Traditional conditional coverage metric~\citep{cauchois2021knowing} considers the worst coverage over all slabs containing $\delta$ mass on the observations, which is defined as
\begin{equation*}
\text{WSC}_n(C,\mathbf{v}):=\inf_{a<b}~ \left\{\mathbb{P}_n(Y\in C(X)\given a \leq \mathbf{v}^TX \leq b) ~~\text{s.t.}~~ \mathbb{P}_n(a \leq \mathbf{v}^TX \leq b) \geq \delta\right\},
\end{equation*}
where $\mathbf{v}\in \mathbb{R}^d$ and $a<b\in \mathbb{R}$.

To strengthen the 
WSC metric, we replace the linear mapping $\mathbf{v}^T$ in the above equation with an 
arbitrary non-linear function $\pi$, giving rise to $\text{WSC}^+$, i.e.,
\begin{equation}\label{Eq:wsc+}
\text{WSC}^+_n(C,\pi):=\inf_{a<b}~ \left\{\mathbb{P}_n(Y\in C(X)\given a\leq \pi(X) \leq b) ~~\text{s.t.}~~ \mathbb{P}_n(a\leq \pi(X) \leq b) \geq \delta\right\}.
\end{equation}
Assume a quadratic function $\pi(\mathbf{x})=\mathbf{x}^T\mathbf{W}\mathbf{x}+\mathbf{v}^T\mathbf{x}$, where $\mathbf{W}\in \mathbb{R}^{d\times d}$ and $\mathbf{v}\in \mathbb{R}^d$. We uniformly draw 1,000 samples $\pi_j=\{\mathbf{W}_j,\mathbf{v}_j\}$ to compute the worst-slab coverage for each $\pi_j$ on the test instances. Following~\cite{cauchois2021knowing}, we use the grid search to achieve the optimal $a,b$ satisfying the desiderata as well.
In this case, we have a lower bound for our metric $\text{WSC}^+$.

\begin{proposition}\label{prop:metric}
Let $\pi(\mathbf{x})=\mathbf{x}^T\mathbf{W}\mathbf{x}+\mathbf{v}^T\mathbf{x}$ be a quadratic function and $\Pi$ be a parameter space of $\pi$. Then, if $C$ effectively provides conditional coverage at level $1-\alpha$, we have 
\begin{equation}\label{Eq:wsc+ bound}
\text{WSC}^+_n=\inf_{\pi\in \Pi}~\text{WSC}^+_n(C, \pi) \geq 1-\alpha-\mathcal{O}(1)\sqrt{\frac{\mathcal{O}(d^2)\log N}{\delta N}}.
\end{equation}
\end{proposition}
The proof is given in Appendix~\ref{A:prop2}.

\begin{figure}[t]
\begin{minipage}{0.45\linewidth}
\centering
\includegraphics[width=0.92\columnwidth]{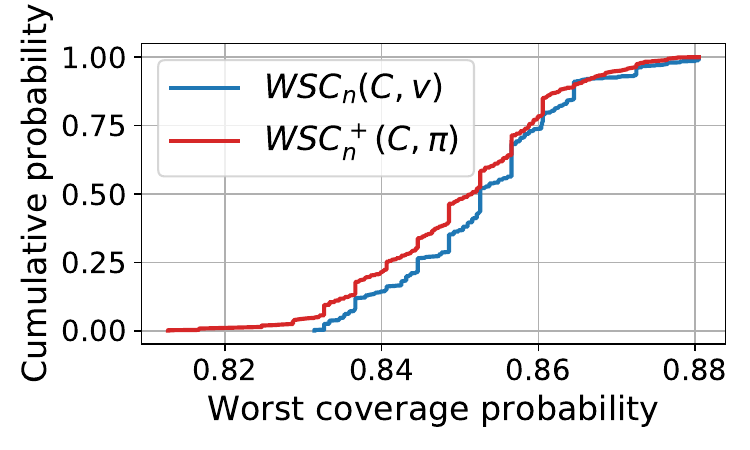}
\vspace{-0.7em}
\caption{
CDF of Conditional Coverage ($\delta=0.5$), which plots the respective cumulative probability curves of different worst-slab coverage discovered by $\text{WSC}_n(C,\mathbf{v})$ and $\text{WSC}_n^+(C,\pi)$ over 1,000 samplings. The red curve is always above the blue curve, indicating that our $\text{WSC}_n^+(C,\pi)$ finds more groups with the poor coverage than $\text{WSC}_n(C,\mathbf{v})$.\label{F:cdf}
}
\end{minipage}
\hfill
\begin{minipage}{0.53\linewidth}
\centering
\vspace{4pt}
\resizebox{\linewidth}{!}{
\begin{sc}
\begin{tabular}{lccccc}
\toprule
{Metric} & $\delta=0.1$ & $\delta=0.2$  & $\delta=0.3$ & $\delta=0.4$ & $\delta=0.5$ \\
 \midrule
\multirow{2}*{$\text{WSC}_n$} & 0.616 & 0.748 & 0.793 & 0.822 & 0.842 \\
    & (0.053) & (0.037) & (0.025) & (0.023) & (0.023) \\
\midrule
\multirow{2}*{$\text{WSC}^+_n$} & 0.582 & 0.674 & 0.750 & 0.800 & 0.829 \\
     & (0.047) & (0.048) & (0.034) & (0.028) & (0.024) \\
\midrule
Imp. & -5.52\% & -9.89\% & -5.42\% & -2.68\% & -1.54\%  \\
\bottomrule
\end{tabular}
\end{sc}}
\vspace{8pt}
\captionof{table}{
Performance of $\text{WSC}_n$ and $\text{WSC}_n^+$ metrics w.r.t. different $\delta$. We repeat the experiment 10 times, and report the average results (the value in () is the standard deviation). Smaller coverage is better. Our metric $\text{WSC}^+_n$ performs better than $\text{WSC}_n$ by up to $9.89\%$ to mine the group with the minimum worst-slab
coverage (defined in Eq.~\ref{Eq:wsc+ bound}).
}
\label{T:class condition distance}
\end{minipage}  
\vspace{-0.8em}
\end{figure}

To demonstrate the advantages of the new metric $\text{WSC}^+$, we randomly draw the features $X\in [0,1]^{10}$ from a uniform distribution and create a simple dataset for classification as described in Appendix~\ref{A:hyperparameter}. 
Note that we define the group needed to be protected to satisfy $(X[0]\geq0.1)\oplus(X[1]\geq0.1)=\text{True}$.
We respectively plot the Cumulative Distribution Functions (CDF) of $\text{WSC}_n(C,v)$ and $\text{WSC}^+_n(C,\pi)$ over 1,000 samples $\pi_j$ when $\delta=0.5$ in Fig.~\ref{F:cdf}, and observe that our $\text{WSC}^+_n(C,\pi)$ always reveals the groups with the worse coverage than that of $\text{WSC}_n(C,v)$, which can be attributed to 
representational capability of the nonlinear function $\pi$ in $\text{WSC}^+_n(C,\pi)$.

Moreover, we also list the average results of two metrics, $\text{WSC}_n$ and $\text{WSC}^+_n$, as $\delta$ increases over 10 repeated experiments in Table~\ref{T:class condition distance}.
Similar to Fig.~\ref{F:cdf}, the minimum worst-slab coverage found by our $\text{WSC}_n^+$ is smaller than that found by  $\text{WSC}_n$ by up to $9.89\%$. As $\delta$ increases, the condition coverage tends to the marginal coverage (0.9), and the gap between $\text{WSC}_n$ and $\text{WSC}_n^+$ narrows, as expected.

\hide{
Moreover, we also list the results of $\text{WSC}_n$ and $\text{WSC}^+_n$ as $\delta$ increases in Table~\ref{T:class condition distance}.
Similar to Fig.~\ref{F:cdf}, the worst-slab coverage found by $\text{WSC}_n^+$ is smaller than that found by  $\text{WSC}_n$ by up to $9.89\%$. As $\delta$ increases, the condition coverage tends to the marginal coverage (0.9), and the gap between $\text{WSC}_n$ and $\text{WSC}_n^+$ narrows, as expected. 
}


\noindent\textbf{Implementations.}\footnote{Our code is publicly available at https://github.com/Xusr1123/FaReG.}
All the experiments are carried out on NVIDIA GeForce RTX 3090. We repeat each experiment 10 times and report the average to suppress randomness. 
We set $\delta=0.5$ for $\text{WSC}_n^+$ by default. More implementation details, such as hyperparameters and training settings, are presented in Appendix~\ref{A:hyperparameter}.

\subsection{Synthetic Data}\label{sec:synthetic data}

We evaluate our method on synthetic data designed to mimic a mental illness diagnosis scenario. The dataset includes six possible labels: Depression, Anxiety Disorders, Bipolar Disorder, Schizophrenia, Anorexia, and Post-Traumatic Stress Disorder (PTSD).
Each sample contains four sensitive features—Age Group, Region, Gender, and Color—along with six non-sensitive features independently sampled from a uniform distribution within a value range $[0,1]$. The sensitive features are generated as follows:
(1) Gender is uniformly drawn from \textit{\{Female, Male\}};
(2) Color is uniformly drawn from \textit{\{Red, Blue\}};
(3) Age Group is drawn from \textit{\{Child, Youth, Middle, Elder\}} with equal probability;
(4) Region follows a fixed cyclical sequence: Asia, Europe, Africa, America, Oceania.

We then generate true labels $Y$ for the dataset, where diagnosis is more challenging for a specific subgroup defined by the Exclusive NOR (XNOR) operation (cf. Appendix~\ref{A:hyperparameter}).
Specifically, we assume $X[0]$ is Color, $X[1]$ is Gender, and $X[2]$ is any non-sensitive feature, and define $Y$ based solely on these three attributes. Through the label generation, we have the following subgroup $X[0]\odot X[1]=\textit{True}$: \textit{Color=Red (True) \& Gender=Female (True)} or \textit{Color=Blue (False) \& Gender=Male (False)}, simulating a real-world situation that algorithmic biases occur on this subgroup.

\begin{figure}[t]
    \centering
    \begin{minipage}[t]{\textwidth}
        \centering
        \begin{subfigure}[t]{0.24\textwidth}
            \centering
            \includegraphics[width=\linewidth]{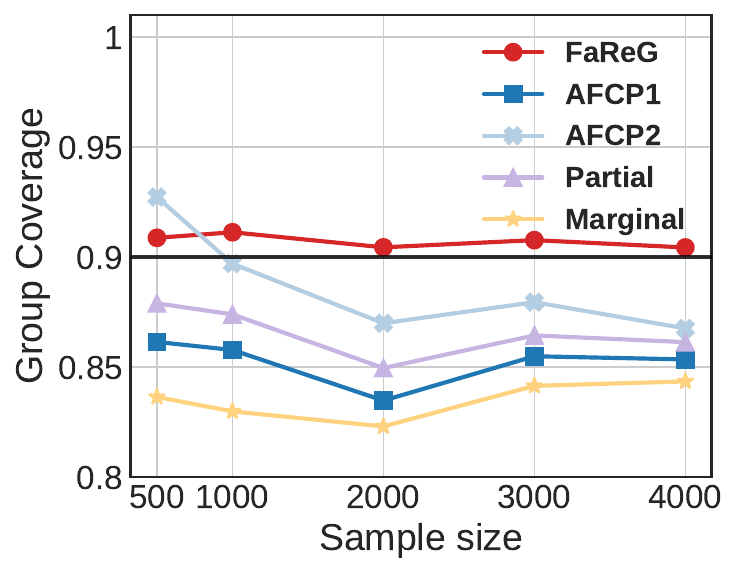}
            \caption{Group Coverage}\label{F:exp1_1}
        \end{subfigure}
        \hfill
        \begin{subfigure}[t]{0.24\textwidth}
            \centering
            \includegraphics[width=\linewidth]{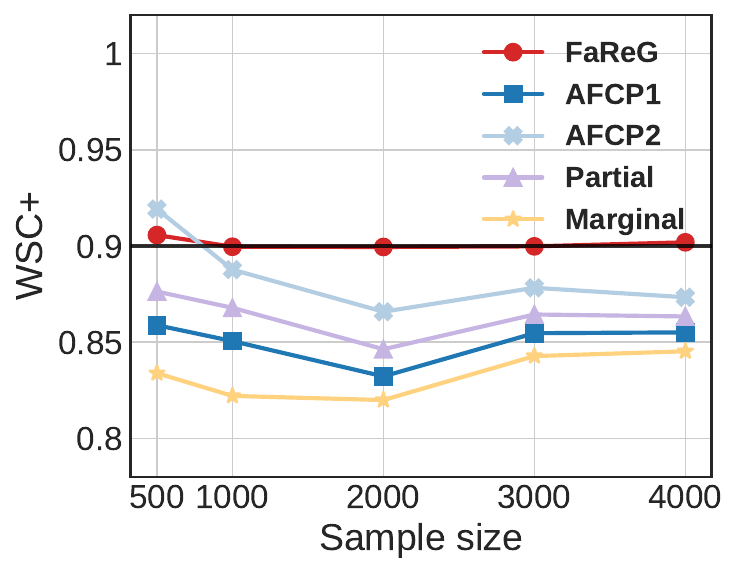}
            \caption{$\text{WSC}^+_n$}\label{F:exp1_2}
        \end{subfigure}
        \hfill
        \begin{subfigure}[t]{0.24\textwidth}
            \centering
            \includegraphics[width=\linewidth]{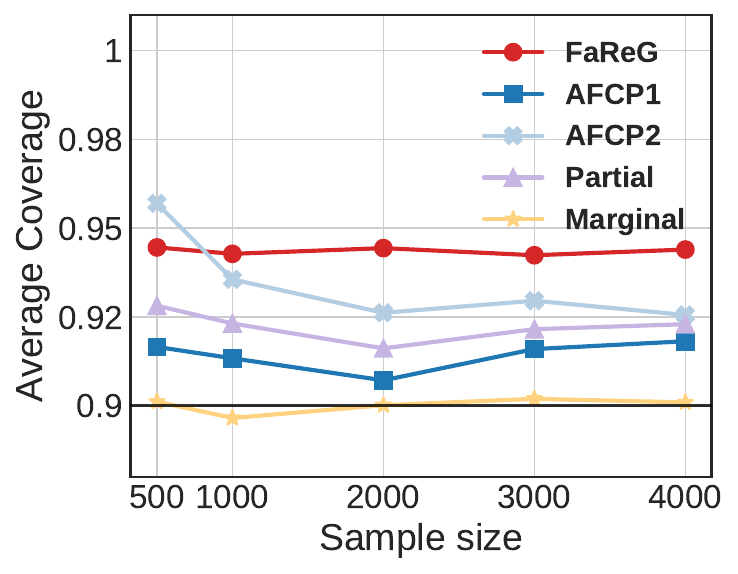}
            \caption{Average Coverage}\label{F:exp1_3}
        \end{subfigure}
        \hfill
        \begin{subfigure}[t]{0.24\textwidth}
            \centering
            \includegraphics[width=\linewidth]{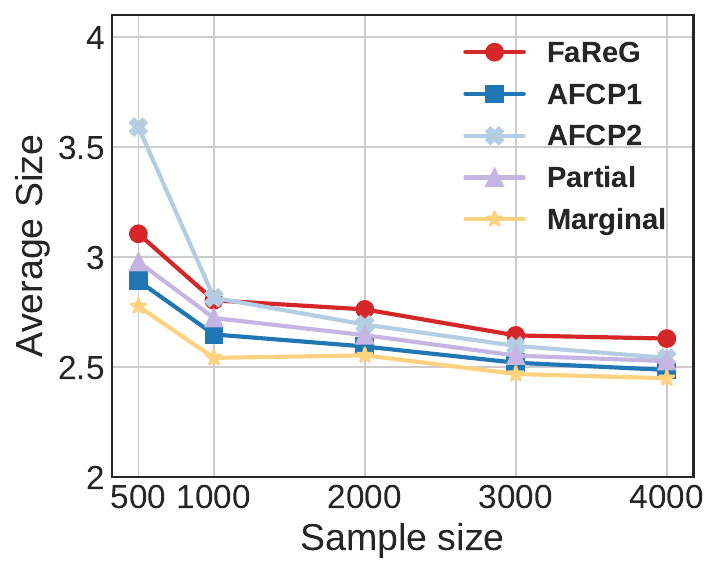}
            \caption{Average Size}\label{F:exp1_4}
        \end{subfigure}
        \vskip -0.5em
        \caption{Performance of prediction sets produced by different CP methods on synthetic data w.r.t. the total number of training and calibration data instances. 
        Only our \name achieves the ideal conditional coverage (0.9), and meanwhile, does not sacrifice too much information (set sizes) compared to baselines.  \label{F:exp1}
        }
    \end{minipage}
\end{figure}

\begin{figure}[t]
    \centering
    \begin{minipage}[t]{\textwidth}
        \centering
        \begin{subfigure}[t]{0.24\textwidth}
            \centering
            \includegraphics[width=\linewidth]{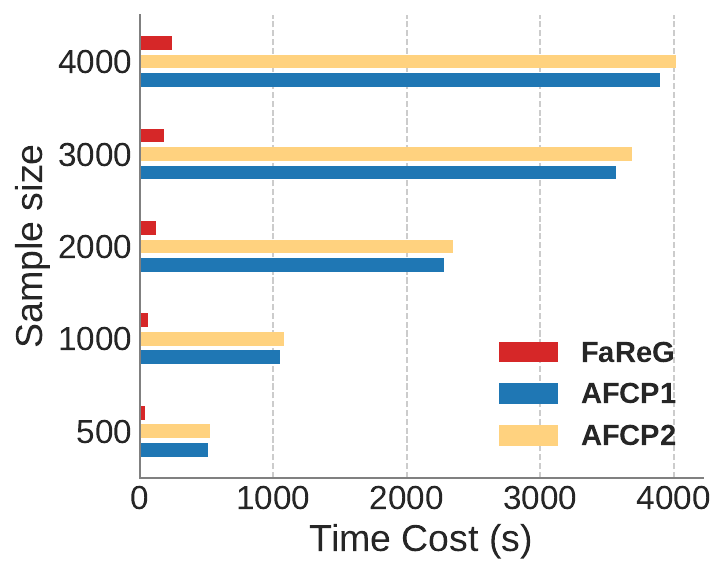}
            \caption{Time Cost}\label{F:time}
        \end{subfigure}
        \hfill
        \begin{subfigure}[t]{0.24\textwidth}
            \centering
            \includegraphics[width=\linewidth]{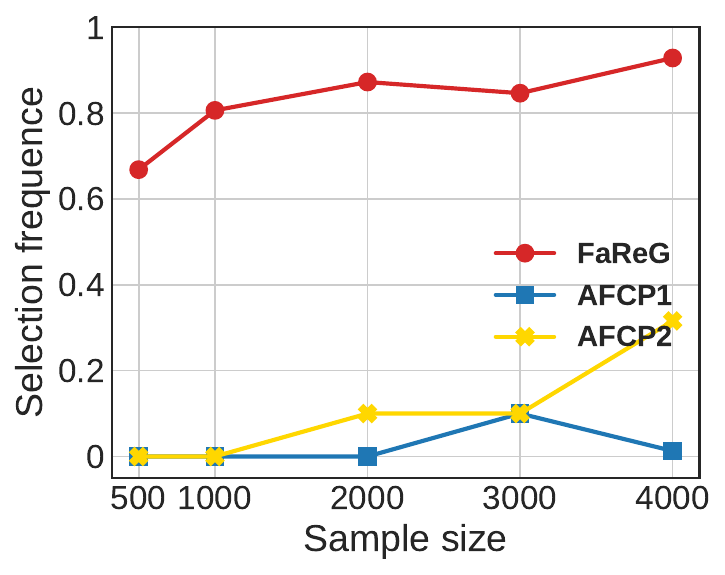}
            \caption{Selects $X[0]$}\label{F:attr8}
        \end{subfigure}
        \hfill
        \begin{subfigure}[t]{0.24\textwidth}
            \centering
            \includegraphics[width=\linewidth]{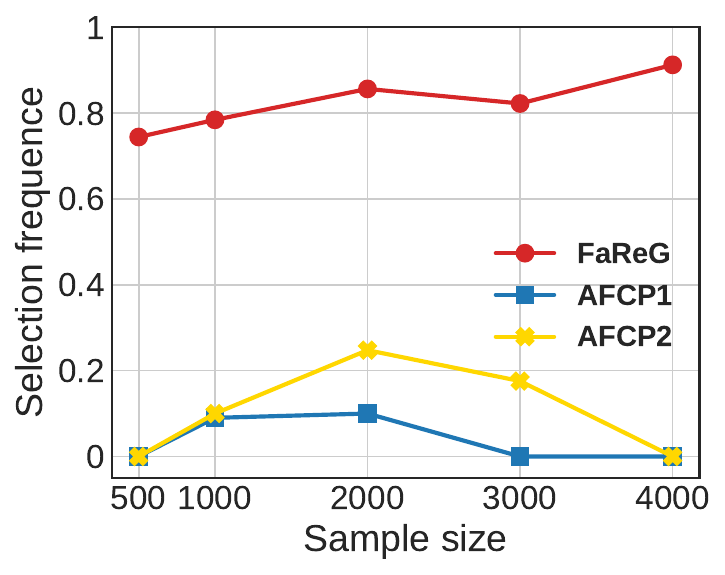}
            \caption{Selects $X[1]$}\label{F:attr9}
        \end{subfigure}
        \hfill
        \begin{subfigure}[t]{0.24\textwidth}
            \centering
            \includegraphics[width=\linewidth]{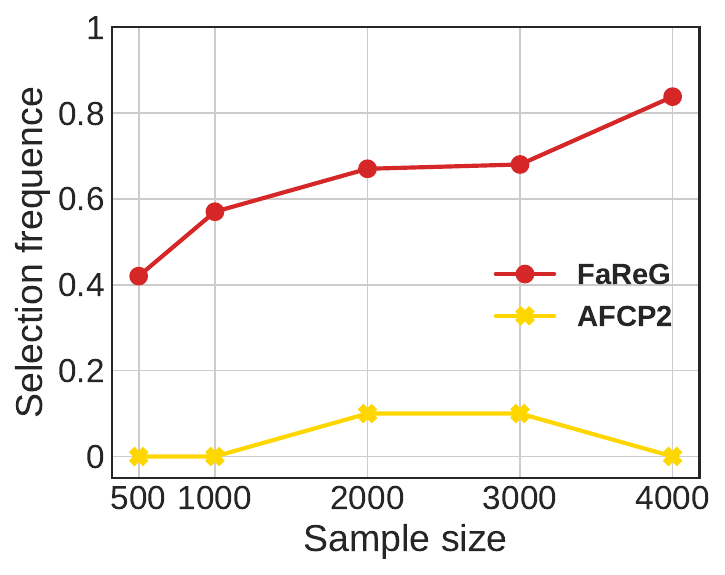}
            \caption{Selects $X[0] \&X[1]$}\label{F:attr8_9}
        \end{subfigure}
        \vskip -0.5em
        \caption{Fig.~(a) reports the running time of different CP methods with the increasing total number of training and calibration data instances. 
        Fig.~(b)--(d) are the results of the selection frequency of target features $X[0]$ and $X[1]$.
        As the sample size increases, our method becomes more consistent with target features.}
        \label{F:exp2}
    \end{minipage}
\end{figure}


\hide{
}

Fig.~\ref{F:exp1} depicts the results of conditional coverage, average coverage (marginal coverage), and average prediction set size (efficiency), respectively.
For conditional coverage, Group Coverage is the coverage on the subgroups defined by XNOR operation as mentioned in data construction, and we compute $\text{WSC}^+_n$ on four predefined sensitive features.
In Fig.~\ref{F:exp1_1} and~\ref{F:exp1_2}, 
our \name is the only one that always achieves valid coverage (greater than 0.9) for the targeted group with varying sample sizes. 
Although the conditional coverage of AFCP2 also exceeds 0.9 when the sample size reaches 500, as shown in Fig.~\ref{F:exp1_4}, it produces considerably larger prediction sets, 
which is less informative for 
decision-making.

\hide{
To further demonstrate the superiority of our method, we report the increase in conditional coverage per efficiency. 
Specifically, in Fig.~\ref{F:exp1_1} and~\ref{F:exp1_4}, we compute the ratio of Group Coverage difference between CP methods,  and Marginal to Efficiency difference (set size) denoted as $\Delta G/\Delta E$. Similarly, we have $\Delta W/\Delta E$ for $\text{WSC}^+_n$. 
The results are listed in Table~\ref{T:C/E}.
We observe that our method outperforms all competitors as $N$ increases, which indicates that \name precisely captures the unfair groups with poor conditional coverage.
}

In Fig.~\ref{F:time}, we compare the average running time of different CP methods (e.g., the wall-clock time of the entire \name pipeline) over 10 repeated experiments, and \name significantly reduces the time cost. 
Actually, the training of the encoder-decoder network occupies most of the time cost (e.g., in one trial with a sample size of 2,000, the training step required 161.3 seconds, while the rest took only 0.8 seconds). Additionally, the training time of the encoder-decoder network scales linearly w.r.t the sample size since we fix the epoch number and batch size in our algorithm.
This result is consistent with the analysis in Section~\ref{sec:Constructing}.

To determine which features are selected by our method, we analyze the predictive variable $S$ and the reconstructed feature $\hat{X}$ by perturbing the latent representation $Z$, following the Beta-VAE approach~\citep{higgins2017beta}. Specifically, we impose a slight perturbation (e.g., $\pm 0.001$) on each dimension of $Z$ and identify the dimension that most influences $S$. Given this influential dimension and prior knowledge (as in AFCP2) that there are exactly two target features, we compute the change ratios for each dimension of $\hat{X}$ before and after perturbation, and select the two features with the top-2 maximum change ratios.

Figures~\ref{F:attr8}, \ref{F:attr9}, and \ref{F:attr8_9} respectively report the frequency of selecting $X[0]$ or $X[1]$ individually, and that of selecting both $X[0]$ and $X[1]$ simultaneously. The results demonstrate that our approach captures more target features than the baselines, and this advantage becomes more pronounced as the sample size increases.

\hide{
To figure out which feature \name actually chooses, we observe the predictive $S$ and reconstruction $\hat{X}$ via changing $z$, following the Beta-VAE method~\citep{higgins2017beta}. Specifically, we manually impose a slight perturbation (e.g., $\pm 0.001$) on each 
dimension of $Z$ and pick the most influential dimension to $S$. Given this dimension of $Z$ and a prior knowledge as AFCP2—there are totally two target features, we compute the differences of $\hat{X}$'s each dimension before and after perturbation, and select two features with the top-$2$ max change ratio.
Fig.~\ref{F:attr8}, \ref{F:attr9} and \ref{F:attr8_9} respectively report the frequency of  selecting $X[0]/X[1]$ solely, and that of simultaneously selecting $X[0]\&X[1]$.
It can be observed that our approach captures more target features than baselines, and this advantage is more evident when the sample size increases.
}

Additionally, we present the results of parameter sensitivity and group visualization in Appendix~\ref{A:Parameter sensitivity} and~\ref{A:vision}, respectively.


\subsection{Nursery Data}\label{sec:nursery data}

\begin{figure}[t]
    \centering
    \begin{minipage}[t]{\textwidth}
        \centering
        \begin{subfigure}[t]{0.24\textwidth}
            \centering
            \includegraphics[width=\linewidth]{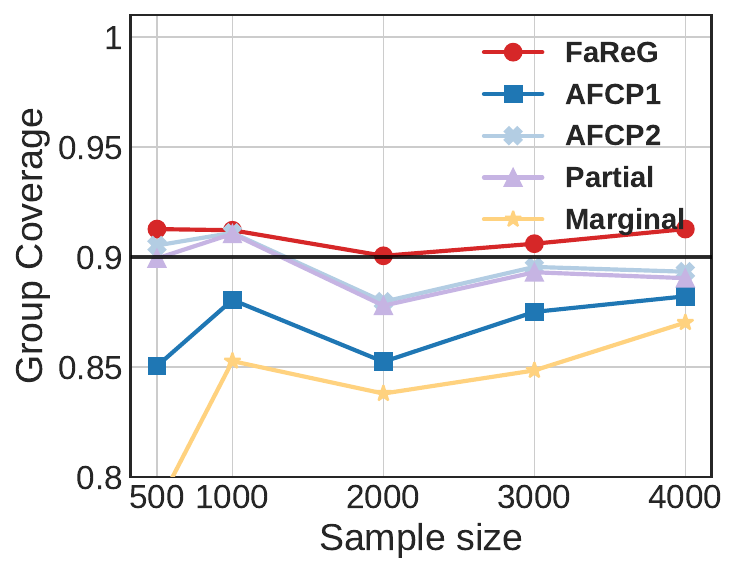}
            \caption{Group Coverage}\label{F:nursery_exp1_1}
        \end{subfigure}
        \hfill
        \begin{subfigure}[t]{0.24\textwidth}
            \centering
            \includegraphics[width=\linewidth]{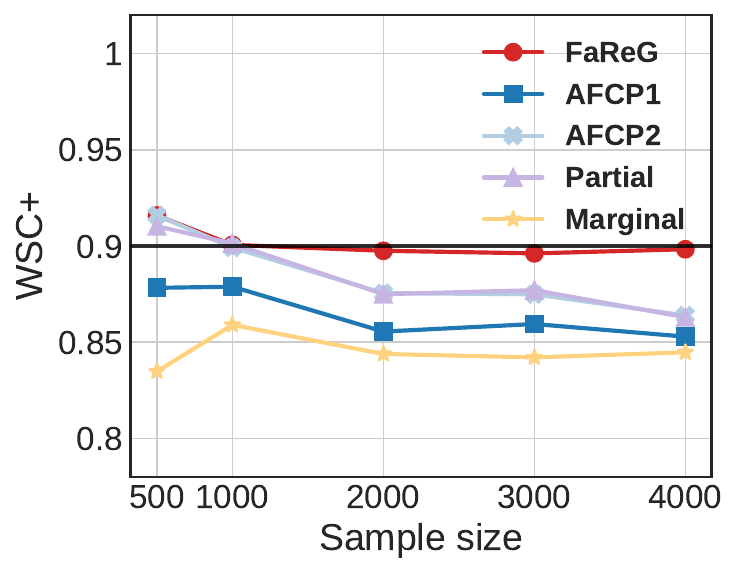}
            \caption{$\text{WSC}^+_n$}\label{F:nursery_exp1_2}
        \end{subfigure}
        \hfill
        \begin{subfigure}[t]{0.24\textwidth}
            \centering
            \includegraphics[width=\linewidth]{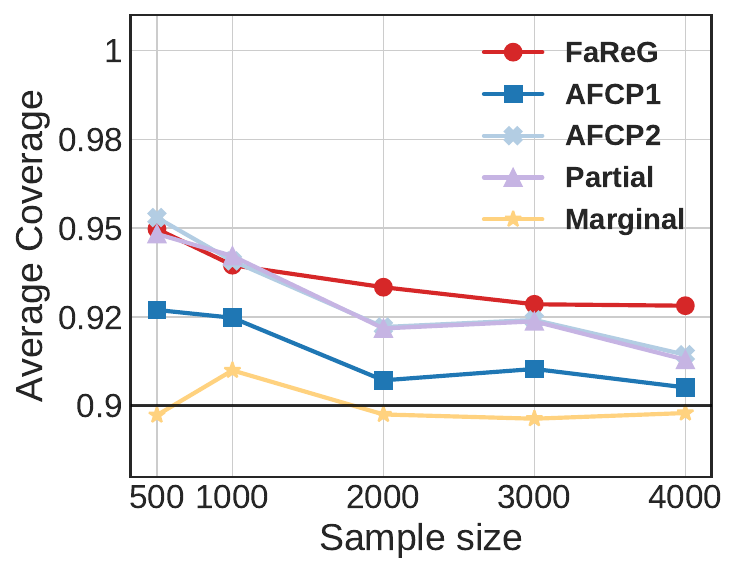}
            \caption{Average Coverage}\label{F:nursery_exp1_3}
        \end{subfigure}
        \hfill
        \begin{subfigure}[t]{0.24\textwidth}
            \centering
            \includegraphics[width=\linewidth]{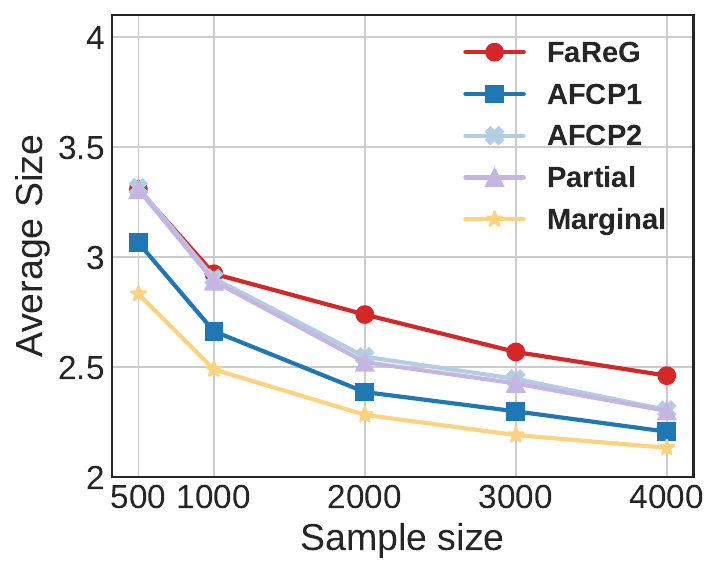}
            \caption{Average Size}\label{F:nursery_exp1_4}
        \end{subfigure}
        \vskip -0.5em
        \caption{Performance of prediction sets produced by different CP methods on the Nursery data w.r.t. the total number of training and calibration data instances. 
        Only our \name achieves the ideal conditional coverage (0.9) and keeps most of the uncertainty information of sets.
        }
        \label{F:nursery_exp2}
    \end{minipage}
\end{figure}

We evaluate our \name and baseline methods on the publicly available Nursery data~\citep{nursery_76}, originally constructed from a hierarchical decision model developed to rank applications for nursery schools. The dataset comprises 12,960 instances, each described by eight categorical features: Parents’ occupation (3 levels, $\textit{Parent:=\{usual, pretentious, great-pret\}}$), Child’s nursery (5 levels), Family form (4 levels), Number of children (4 levels), Housing conditions (3 levels), Financial standing (2 levels, $\textit{Finance:=\{convenient, inconv\}}$), Social conditions (3 levels) and Health status (3 levels). The task is to classify applications into one of five priority ranks. We take all features into account (as sensitive features) except Housing conditions.

In data preprocessing, we strictly follow~\citet{zhou2024conformal}, and consider a group defined by $\textit{Parent=usual } \& \textit{ Finance=inconv}$ or $\textit{Parent=pretentious } \& \textit{ Finance=inconv}$. 
To make the issue more interesting and control the degree of algorithmic bias, we corrupt the labels of instances in such a group by adding independent, uniform noise and rounding to the nearest integer (label) as similar as \cite{zhou2024conformal}. 
This perturbation amplifies the intrinsic unpredictability of the group defined before, thereby increasing its vulnerability to algorithmic bias.

Fig.~\ref{F:nursery_exp2} presents the results. Our method consistently achieves the valid coverage under both conditional coverage metrics, i.e., Group Coverage and $\text{WSC}^+_n$, outperforming all baselines. Partial and AFCP2 perform better than the other CP methods, but \name still achieves superior results.

\section{Related Work} \label{sec:rel}

Conformal Prediction (CP) has seen vigorous development in recent years~\citep{vovk2005algorithmic,smith2024uncertainty}.
Its applications span diverse domains, from image classification~\citep{sadinle2019least} and object detection~\citep{teng2022predictive} to large language models~\citep{kumar2023conformal}.

Some CP work, building on the split conformal framework~\citep{papadopoulos2002inductive, lei2018distribution}, introduces advanced nonconformity scores to ensure valid marginal coverage on the empirical data distribution. For example, \citet{romano2019conformalized} gives a nonconformity score based on quantile regression, while \citet{romano2020classification} and \citet{angelopoulos2020uncertainty} design nonconformity scores for classification.
Additionally, \citet{hoff2023bayes} proposes a nonconformity score to achieve Bayes optimal coverage.

Another line of work has explored various notions of equalized coverage~\citep{romano2020malice} and empirically evaluated the corresponding conformal predictors in real-world applications~\citep{lu2022fair}.
For regression tasks, \citet{wang2023equal} guarantees equal coverage rates across more fine-grained groups on continuous features, and 
\citet{liu2022conformalized} propose to learn a real-valued quantile function with respect to sensitive features.
They address a distinct notion of equalized coverage tailored to continuous outcomes.
In classification, label-conditional coverage is a common alternative to equalized coverage~\citep{vovk2003mondrian,lofstrom2015bias,ding2023class}. 
This work defines the groups to be protected based on the label $Y_{N+1}$, instead of the features $X_{N+1}$.
\citet{jung2022batch}, \citet{vadlamani2025generic},    and 
\citet{gibbs2025conformal} adopt group-conditional coverage, which is analogous to equalized coverage, to improve prediction sets.
Different from the previous work, our approach \name can adaptively identify unfairly treated groups without the assumption that such groups are pre-defined.
AFCP~\citep{zhou2024conformal} develops an algorithm to construct CP sets with valid equalized coverage for adaptively selected groups, which establishes the current state-of-the-art for equalized coverage tasks.
\section{Conclusion} \label{sec:con}

 In this paper, we propose \name, a fair conformal prediction method that learns latent groups to achieve adaptive equalized coverage. By leveraging a variational encoder-decoder to discover subgroups wih poor coverage in a high-level feature space, our approach captures complex algorithmic biases that linear methods may neglect and can be adapted to different fairness notions. 
 We also propose $\text{WSC}^+$, a nonlinear metric for evaluating the conditional coverage of unfair groups more accurately. 
 Extensive experiments confirm that \name efficiently offers stronger fairness guarantees, showing a more expressive and practical path toward fair, reliable conformal inference.

\noindent\textbf{Limitations.} 
The enhanced expressivity of representation-based groups may 
sacrifice model interpretability partially, compared to groups explicitly defined on manifest features. 
However, the encoder-decoder structure compensates this shortcoming well via reconstructing the input $X$, which is empirically confirmed by Section~\ref{sec:synthetic data} and Appendix~\ref{A:vision}.



\section*{Acknowledgment}
We appreciate anonymous reviewers for their valuable comments.  
This work is supported by the Frontier Technologies R\&D Program of Jiangsu (BF2024059), National Natural Science Foundation of China (Grants \#62025202), and the Collaborative Innovation Center of Novel Software Technology and Industrialization.
T. Chen is partially supported by overseas grants from the State Key Laboratory of Novel Software Technology, Nanjing University (KFKT2023A04,  KFKT2025A05).
Yuan Yao is the corresponding author.









\bibliography{refs}
\bibliographystyle{iclr2026_conference}

\appendix
\newpage


\section{Technical proofs}

\subsection{Variational Inference} \label{appendix:var}

As mentioned in Section~\ref{sec:learning groups}, our optimization objective is as follows, 
\begin{equation*}
\max I(Z,S) + I(Z,X) - \beta I(Z,i).
\end{equation*}

First of all, we consider $I(Z,S)$ and 
\begin{equation}\label{Eq:I(Z,S)-1}
I(Z,S) = \int p_\theta(s,z)\log \frac{p_\theta(s, z)}{p_\theta(s)p_\theta(z)} \d s \d z = \int p_\theta(s,z)\log \frac{p_\theta(s \given  z)}{p_\theta(s)} \d s \d z.
\end{equation}

Since the KL divergence between two conditional probability distribution $p_\theta(s \given  z)$ and $q_\phi(s \given  z)$ is non-negative, we have
\begin{equation*}
D_{\mathrm{KL}}(p_\theta(s \given  z) \Vert q_\phi(s \given  z)) \geq 0 \Rightarrow \int p_\theta(s,z)\log p_\theta(s \given  z) \d s \geq \int p_\theta(s,z)\log q_\phi(s\given  z) \d s,
\end{equation*}
where $q_\phi(s\given  z)$ is a variational approximation to the intractable distribution $p_\theta(s \given  z)$.

Plugging the above inequality into Eq.~\ref{Eq:I(Z,S)-1}, we obtain
\begin{equation}\label{Eq:I(Z,S)-2}
\begin{aligned}
I(Z,S) & \geq \int p_\theta(s,z)\log \frac{q_\phi(s \given  z)}{p_\theta(s)} \d s \d z\\
&= \int p_\theta(s,z) \log q_\phi(s \given  z) \d s \d z + \int p_\theta(s) \log p_\theta(s) \d s \\
& \ge \int p_\theta(s,z) \log q_\phi(s \given  z) \d s \d z,
\end{aligned}
\end{equation}
where the second inequality is derived by the non-negativity of entropy.

Since $S\upmodels Z\given X$ holds, we have
\begin{equation*}
p_\theta(s,z)=\int p_\theta(x,s,z) \d x=\int p_\theta(x)p_\theta(s\given  x)p_\theta(z\given  x) \d x.
\end{equation*}
Hence, we get
\begin{equation}\label{Eq:I(Z,S)-3}
I(Z,S) \geq \int p_\theta(x)p_\theta(s\given  x)p_\theta(z\given  x)\log q_\phi(s\given  z) \d x\d s\d z.
\end{equation}

Similar to Eq.~\ref{Eq:I(Z,S)-2}, we also have
\begin{equation}\label{Eq:I(Z,X)}
\begin{aligned}
I(Z,X) &\geq \int p_\theta(x,z)\log q_\varphi(x\given  z) \d x\d z \\ 
&= \int p_\theta(x)p_\theta(z\given  x)\log q_\varphi(x\given  z) \d x\d z.\\
\end{aligned}
\end{equation}

As for $I(Z,i)$, we have
\begin{equation}\label{Eq:I(Z,i)}
\begin{aligned}
I(Z,i) &= \sum_i \int p_\theta(z \given  i)p_\theta(i)\log \frac{p_\theta(z\given  i)}{p_\theta(z)}\d z \\ 
&= \frac{1}{N}\sum_i\int p_\theta(z\given  x_i)\log \frac{p_\theta(z\given  x_i)}{p_\theta(z)}\d z\\
&\leq \frac{1}{N}\sum_i\int p_\theta(z\given  x_i)\log \frac{p_\theta(z\given  x_i)}{r(z)}\d z,
\end{aligned}
\end{equation}
where $r(z)$ is a variational approximation to the posterior distribution $p_\theta(z)$. We usually set $r(z)$ as a standard normal distribution $\mathcal{N}(0,1)$ in practice.

Combining Eq.~\ref{Eq:I(Z,S)-3} with Eq.~\ref{Eq:I(Z,X)} and Eq.~\ref{Eq:I(Z,i)}, we obtain

\begin{equation*}
\begin{aligned}
& I(Z,S) + I(Z,X) - \beta I(Z,i) \geq \int p_\theta(x)p_\theta(s\given  x)p_\theta(z\given  x)\log q_\phi(s\given  z) ~\d x\d s\d z \\
&\qquad + \int p_\theta(x)p_\theta(z\given  x)\log q_\varphi(x\given  z) ~\d x\d z - \frac{\beta}{N}\sum_i\int p_\theta(z\given  x_i)\log \frac{p_\theta(z\given  x_i)}{r(z)}\d z.
\end{aligned}
\end{equation*}

With Monte Carlo sampling, we use the empirical dataset on $\{X_i, S_i, Y_i\}_{i=1}^N$ to estimate $p_\theta(x)p_\theta(s\given x)$ and $p_\theta(x)$, 
where $S_i$ is computed by minimizing the conditional coverage of groups defined by $\mathbf{S}=\{s_1,\dots,s_N\}$, i.e., $\mathbb{P}_n[Y_i \in C(X_i) \given X_i \in \hat{\mathcal{G}}_\mathbf{S}]$ on $\{X_i, S_i, Y_i\}_{i=1}^N$.
We leverage the reparameterization trick~\citep{kingma2013auto} as mentioned in Section~\ref{sec:learning groups}, and finally obtain
\begin{equation*}
\mathcal{L} = -\frac{1}{N}\sum_{i=1}^N \left( \mathbb{E}_{\tilde{z}\sim f(x_i,\epsilon)}[\log q_\phi(s_i\given  \tilde{z}) + \log q_\varphi(x_i\given  \tilde{z})] - \beta D_{\mathrm{KL}}(p_\theta(z\given  x_i)\| r(z))\right).
\end{equation*}

\subsection{Proof of Proposition~\ref{prop:finite}} \label{A:prop1}


\begin{proof}
We first present a technical lemma, 
where $P_nh=\frac{1}{N}\sum_{i=1}^Nh(X_i)$ and $Ph=\int h(x)dP(x)$, given an observed dataset $\{X_i,Y_i\}_{i=1}^N$.
\begin{lemma}[\cite{boucheron2005theory}]
There exists a numerical constant $C_1$ such that for any $\tau > 0$, 
\begin{equation*}
\vert P_nh -Ph \vert \leq C_1\left[\sqrt{\min\{P_nh,Ph\}\frac{VC(h)\log N+\tau}{N}} + \frac{VC(h)\log N+\tau}{N}\right]
\end{equation*} 
holds with probability at least $1-e^{-\tau}$.
\end{lemma}

By this Lemma, we have
\begin{align}
&\lvert P_n(Y \in C(X), X \in \hat{\mathcal{G}}_\mathbf{s}) 
   - P(Y \in C(X), X \in \hat{\mathcal{G}}_\mathbf{s}) \rvert  \label{Eq:p1-1} \\
& \qquad \leq C_1\left[\sqrt{\min\{P_n(Y \in C(X), X \in \hat{\mathcal{G}}_\mathbf{s}),
   P(Y \in C(X), X \in \hat{\mathcal{G}}_\mathbf{s})\}\tfrac{VC(h)\log N+\tau}{N}}
   + \tfrac{VC(h)\log N+\tau}{N}\right]. \notag
\end{align}
Similarly, we get
\begin{align}
& \lvert P_n(X \in \hat{\mathcal{G}}_\mathbf{s}) - P(X \in \hat{\mathcal{G}}_\mathbf{s}) \rvert \label{Eq:p1-2} \\
& \qquad \leq C_2\left[\sqrt{\min\{P_n(X \in \hat{\mathcal{G}}_\mathbf{s}), P(X \in \hat{\mathcal{G}}_\mathbf{s})\}\tfrac{VC(h)\log N+\tau}{N}} + \tfrac{VC(h)\log N+\tau}{N}\right] \notag.
\end{align}
Then, it remains to show that 
\begin{equation*}
\begin{aligned}
&\lvert P_n(Y \in C(X) \given X \in \hat{\mathcal{G}}_\mathbf{s}) - P(Y \in C(X) \given X \in \hat{\mathcal{G}}_\mathbf{s}) \rvert \\
&\qquad = \left\lvert \frac{P_n(Y \in C(X), X \in \hat{\mathcal{G}}_\mathbf{s})}{P_n(X \in \hat{\mathcal{G}}_\mathbf{s})} - \frac{P(Y \in C(X), X \in \hat{\mathcal{G}}_\mathbf{s})}{P(X \in \hat{\mathcal{G}}_\mathbf{s})} \right\rvert.
\end{aligned}
\end{equation*}
Let $a=P_n(Y \in C(X), X \in \hat{\mathcal{G}}_\mathbf{s}), b=P(Y \in C(X), X \in \hat{\mathcal{G}}_\mathbf{s}),c=(P_n-P)(X \in \hat{\mathcal{G}}_\mathbf{s})$ and $d=P(X \in \hat{\mathcal{G}}_\mathbf{s})$. 
We can derive $b\leq d$, and observe that 
\begin{equation}\label{Eq:p1-3}
\lvert \frac{a}{c+d} - \frac{b}{d} \rvert \leq  \lvert \frac{a}{c+d} - \frac{b-c }{c+d} \rvert \leq \frac{\lvert a-b \rvert}{c+d} + \frac{\lvert c \rvert}{c+d}.
\end{equation}
Substitute Eq.~\ref{Eq:p1-1} and Eq.~\ref{Eq:p1-2} into Eq.~\ref{Eq:p1-3}, and use $\delta=P_n(X \in \hat{\mathcal{G}}_\mathbf{s})$, we obtain
\begin{equation*}
\begin{aligned}
&\left\lvert \frac{P_n(Y \in C(X), X \in \hat{\mathcal{G}}_\mathbf{s})}{P_n(X \in \hat{\mathcal{G}}_\mathbf{s})} - \frac{P(Y \in C(X), X \in \hat{\mathcal{G}}_\mathbf{s})}{P(X \in \hat{\mathcal{G}}_\mathbf{s})} \right\rvert \\
&\quad \leq \frac{\lvert P_n(Y \in C(X), X \in \hat{\mathcal{G}}_\mathbf{s}) - P(Y \in C(X), X \in \hat{\mathcal{G}}_\mathbf{s}) \rvert}{P_n(X \in \hat{\mathcal{G}}_\mathbf{s})} - \frac{\lvert P_n(X \in \hat{\mathcal{G}}_\mathbf{s}) - P(X \in \hat{\mathcal{G}}_\mathbf{s}) \rvert}{P_n(X \in \hat{\mathcal{G}}_\mathbf{s})}\\
&\quad \leq C_3\left[\sqrt{\frac{VC(h)\log N+\tau}{\delta N}} + \frac{VC(h)\log N+\tau}{\delta N}\right],
\end{aligned}
\end{equation*}
which completes the proof.
\end{proof}


\subsection{Optimization Process of Eq.~\ref{Eq:project}} \label{A:opt}
The projection operation of the PGD algorithm described in Section~\ref{sec:learning groups} requires solving the following optimization to minimize the $\ell_2$ distance:
\begin{equation}\label{Eq:project problem}
\min_{v_1,\dots,v_n} \sum_{i=1}^N (v_i - u_i)^2 \quad \text{s.t.} ~\sum_{i=1}^N v_i \geq \delta, \quad v_i \in [0,1] ~~i=1,\dots,N,
\end{equation}
where $u_1,\dots,u_N$ are given and $u_i\in[0, 1]$ holds for each $i\in [N]$.

With the above constraints, we compute the Lagrangian as
\begin{equation*}\label{Eq:L2}
\mathcal{L}(v_i; \lambda_i, \mu_i, \omega) = \sum_{i=1}^N(v_i-u_i)^2 + \sum_{i=1}^N\lambda_i(-v_i) + \sum_{i=1}^N\mu_i(v_i-1) + \omega(\delta-\sum_{i=1}^Nv_i),
\end{equation*}
where $\{\lambda_i\}_{i=1}^N,\{\mu_i\}_{i=1}^N$ and $\omega$ are the Lagrange multipliers. 
Let the partial derivatives vanish, and we have
\begin{equation*}
\frac{\partial \mathcal{L}}{\partial v_i} = 2(v_i - u_i) - \lambda_i + \mu_i - \omega = 0 \Rightarrow 2(v_i - u_i) = \lambda_i - \mu_i + \omega
\end{equation*}
For the complementary relaxation conditions, there are four different cases:
\begin{itemize}
    \item If $v_i=0$, constraint $v_i\geq 0$ is activated and we have $\lambda_i\geq 0,\mu_i=0$;
    \item If $v_i=1$, constraint $v_i\leq 1$ is activated and we have $\mu_i\geq 0,\lambda_i=0$;
    \item If $0<v_i<1$, we have $\mu_i =\lambda_i=0$ and then $v_i=u_i+\omega/2$;
    \item If $\sum v_i> \delta$, constraint $\sum v_i\geq \delta$ is not activated and then $\omega=0$; otherwise, $\omega\geq 0$.
\end{itemize}

When $\sum u_i \geq \delta$, we have $v_i=u_i$, which is an optimal solution to the minimization problem in Eq.~\ref{Eq:project problem}. 

When $\sum u_i < \delta$, let $v_i = \min(1,u_i+\omega/2)$, where $\omega\geq 0$ and $\sum_{i=1}^N\min(1,u_i+\omega/2) \geq \delta$. In this case, we resort $\{v_i\}_{i=1}^N$ in descending order, i.e., $v_{(1)}\geq v_{(2)}\geq \dots\geq v_{(N)}$. 
Let $k\in [N]$ is the greatest index to satisfy $v_{(k)}+\omega/2\geq 1$ and $v_{(k+1)}+\omega/2< 1$. Then, constraint $\sum v_i= \delta$ can be written as
\begin{equation*}
k\cdot1+\sum_{i=k+1}^N(v_{(i)}+\omega/2)=\delta.
\end{equation*}
Hence, we obtain
\begin{equation*}
\omega=\frac{2(\delta-k-\sum_{i=k+1}^Nv_{(i)})}{N-k}.
\end{equation*}

In practice, we can compute $k$ and $\omega$ via traversing the value of $k$ from maximum $N$ to minimum $1$.


\subsection{Proof of Theorem~\ref{thm:equalized coverage}}\label{A:thm}

\begin{proof}
When making the similar assumption as Theorem 1 in AFCP~\citep{zhou2024conformal}, for each group $\hat{\mathcal{G}}_{\mathbf{s}} \in \{\hat{\mathcal{G}}_{\mathbf{s}_t}\}_{t=1}^T$, we can substitute $X_{N+1}\in\hat{\mathcal{G}}_{\mathbf{s}}$ for $\phi(X_{N+1},\hat{A}(X_{N+1}))$ and $X_{N+1}\in\hat{\mathcal{G}}^o_{\mathbf{s}}$ for $\phi(X_{N+1},\hat{A}^o(X_{N+1}))$ as conditions, where $\hat{\mathcal{G}}^o_{\mathbf{s}}$ is an imaginary oracle group.
Then, according to Theorem 1~\citep{zhou2024conformal}, we have 
\begin{equation*}
\mathbb{P}[Y_{N+1} \in C(X_{N+1}) \given X_{N+1}\in \hat{\mathcal{G}}_{\mathbf{s}}] \geq 1-\alpha.
\end{equation*}

AFCP assumes that the group selection algorithm can always achieve the oracle group $\hat{\mathcal{G}}^o_{\mathbf{s}}$, which means that the algorithm must have enough expressiveness to include $\hat{\mathcal{G}}^o_{\mathbf{s}}$ into the candidate group space.
However, this necessary condition could be violated, as AFCP's candidate group space is limited to linear groups defined by individual features.
In contrast, our method, \name, employs a more expressive model that extends its candidate group space into the nonlinear realm. Consequently, the guarantee for \name remains valid for groups defined by complex, nonlinear feature combinations.

Next, we formally analyze the expressiveness of AFCP and our \name based on the VC-dimension.
As described in Section~\ref{sec:intro}, AFCP computes the group coverage scores for each feature and greedily picks the most sensitive feature with the lowest group coverage score.
The essence of such a process is a decision stump dividing all features into two parts (sensitive or not sensitive) using a threshold, and thus its VC-dimension is 2.
In contrast, based on established theory~\citep{shalev2014understanding}, the VC-dimension of \name scales with its parameter size $M$, i.e.,
\begin{equation*}
\text{VC}(\text{AFCP})=2, \quad \text{VC}(\text{\name})=\mathcal{O}(M).
\end{equation*}
Hence, the VC-dimension of our \name is typically far larger than that of AFCP, indicating the stronger expressiveness of our method, i.e., our candidate group space serves as a superset of AFCP's candidate group space.
\end{proof}

\subsection{Proof of Proposition~\ref{prop:metric}}
\label{A:prop2}


\begin{proof}
According to Proposition~\ref{prop:finite} and the definition of $\text{WSC}^+_n$ (Eq.~\ref{Eq:wsc+}), we obtain
\begin{equation*}
\sup_{\pi\in \Pi} \{\lvert \text{WSC}^+_n(C, \pi) - \mathbb{P}(Y\in C(X)\given a\leq \pi(X) \leq b)\rvert\}
\leq \mathcal{O}(1)\sqrt{\frac{VC(\Pi)\log N}{\delta N}}
\end{equation*}
by omitting $\tau$.
Then, we eliminate the absolute value as
\begin{equation*}
-\mathcal{O}(1)\sqrt{\frac{VC(\pi)\log N}{\delta N}} \leq \text{WSC}^+_n(C, \pi) - \mathbb{P}(Y\in C(X)\given a\leq \pi(X) \leq b)
\leq \mathcal{O}(1)\sqrt{\frac{VC(\pi)\log N}{\delta N}},
\end{equation*}
which holds for all $\pi\in\Pi$.
Hence, if $\mathbb{P}(Y\in C(X)\given a\leq \pi(X) \leq b) = 1-\alpha$, we can observe
\begin{equation*}
\text{WSC}^+_n(C, \pi) \geq \mathbb{P}(Y\in C(X)\given a\leq \pi(X) \leq b) - \mathcal{O}(1)\sqrt{\frac{VC(\pi)\log N}{\delta N}}
\end{equation*}
for any $\pi\in\Pi$.

Next, we only need to prove $\mathcal{O}(d^2) \geq VC(\pi)$.
Recall that $VC(\pi)$ denotes the VC-dimension of the binary classifier $\pi$, and $\pi=\mathbf{x}^T\mathbf{W}\mathbf{x}+\mathbf{v}^T\mathbf{x}$ is a quadratic function, where $\mathbf{W}\in \mathbb{R}^{d\times d}$ and $\mathbf{v}\in \mathbb{R}^d$.
Therefore, the VC-dimension of $\pi$ is equal to the dimension of its expanded feature space $\mathcal{M}=d(d+1)/2 +d$, i.e., $\mathcal{O}(d^2)$,
which completes the proof.
\end{proof}


\section{Further experiment details}

\subsection{Dataset Construction and Hyperparameters}\label{A:hyperparameter}

 

\begin{table*}[h]
    \centering
    \caption{Hyperparameters of \name.   \label{T:hyperparameters}}
    \resizebox{0.7\linewidth}{!}{
    \begin{sc}
    \begin{threeparttable}
      \begin{tabular}{lcc}
       \toprule
        {Dataset} & Synthetic Data & Nursery Data   \\
        \midrule
        {Model} & MLP & MLP \\
        {Number of Layers}& 3 & 3 \\
        {Hidden dimension}& [64,32] & [64,32] \\
        {Epoch}& 2000 & 800 \\
        {Batch Size}& 500 & 500 \\
        {Learning Rate}& 0.001 & 0.01 \\
        {$\beta$}& 2.0 & 0.1 \\
        {$\delta$}& 0.3 & 0.1 \\
        {$T$}& 20 & 100 \\
      \bottomrule
      \end{tabular}
    \end{threeparttable}
    \end{sc}
 }
\end{table*}

For the dataset we use to evaluate two metrics in Section~\ref{sec:settings}, only $X[0],X[1],$ and $X[2]$ influence the label $Y$ and we define the conditional distribution $P(Y\given X)$ as
\begin{equation*}
P(Y \given X) = 
\begin{cases} 
\left( \frac{1}{3}, \frac{1}{3}, \frac{1}{3}, 0, 0, 0 \right), 
& \text{if } (X[0]\geq0.1)\oplus(X[1]\geq0.1) \text{ and } X[2] < 0.5, \\
\left( 0, 0, 0, \frac{1}{3}, \frac{1}{3}, \frac{1}{3} \right), 
& \text{if } (X[0]\geq0.1)\oplus(X[1]\geq0.1) \text{ and } X[2] \geq 0.5, \\
\left( 1, 0, 0, 0, 0, 0 \right), 
& \text{if not } (X[0]\geq0.1)\oplus(X[1]\geq0.1) \text{ and } X[2] < \frac{1}{6}, \\
\left( 0, 1, 0, 0, 0, 0 \right), 
& \text{if not } (X[0]\geq0.1)\oplus(X[1]\geq0.1) \text{ and } \frac{1}{6} \leq X[2] < \frac{2}{6}, \\
\vdots \\
\left( 0, 0, 0, 0, 0, 1 \right), 
& \text{if not } X[0] = (X[0]\geq0.1)\oplus(X[1]\geq0.1) \text{ and } \frac{5}{6} \leq X[2] \leq 1.
\end{cases}   
\end{equation*}

For the classification models as the input of conformal prediction, we strictly follow the settings in~\citep{zhou2024conformal} on both synthetic and real-world data.
To train \name to mine unfair groups, we randomly split the calibration set $\mathcal{D}$ into the training set and validation set with the ratio 5:5.
We list the hyperparameters of \name in Table~\ref{T:hyperparameters}.
Note that we use the same network structure for encoders and decoders, i.e., a simple 3-layer MLP, which is consistent with Proposition~\ref{prop:finite}.
Since there are three optimization objectives in Eq.~\ref{Eq:Loss Funciton_2}, which may conflict with each other to some extent, we divide the training into two stages. At the first stage, we train the encoder with the parameter $\theta$ and the decoder with the parameter $\phi$ by fixing the decoder with the parameter $\varphi$ in practice, i.e., the first term $\mathcal{L}_{CC}$ and third term $\mathcal{L}_{KL}$ in Eq.~\ref{Eq:Loss Funciton_2}.
Then, we use $\mathcal{L}_{MSE}$ to reconstruct $X$ based on $Z$ at the second stage.

Recall from Section~\ref{sec:synthetic data} that Color is denoted as $X[0]$, Gender is denoted as $X[1]$, and the first standard feature is denoted as $X[2]$. The conditional distribution of $Y\given X$ is determined by a simple decision tree, where only $X[0]$, $X[1]$, and $X[2]$ provide valuable predictive information for $Y$, formulated as follows,

\begin{equation*}
P(Y \given X) = 
\begin{cases} 
\left( \frac{1}{3}, \frac{1}{3}, \frac{1}{3}, 0, 0, 0 \right), 
& \text{if } X[0] = \textit{Red } \text{and } X[1] = \textit{Female } \text{and } X[2] < 0.5, \\
\left( 0, 0, 0, \frac{1}{3}, \frac{1}{3}, \frac{1}{3} \right), 
& \text{if } X[0] = \textit{Red } \text{and } X[1] = \textit{Female } \text{and } X[2] \geq 0.5, \\
\left( \frac{1}{3}, \frac{1}{3}, \frac{1}{3}, 0, 0, 0 \right), 
& \text{if } X[0] = \textit{Blue } \text{and } X[1] = \textit{Male } \text{and } X[2] < 0.5,\\
\left( 0, 0, 0, \frac{1}{3}, \frac{1}{3}, \frac{1}{3} \right), 
& \text{if } X[0] = \textit{Blue } \text{and } X[1] = \textit{Male } \text{and } X[2] \geq 0.5,\\
\left( 1, 0, 0, 0, 0, 0 \right), 
& \text{if } X[0] = \textit{Red } \text{and } X[1] = \textit{Male } \text{and } X[2] < \frac{1}{6}, \\
\left( 0, 1, 0, 0, 0, 0 \right), 
& \text{if } X[0] = \textit{Red } \text{and } X[1] = \textit{Male } \text{and } \frac{1}{6} \leq X[2] < \frac{2}{6}, \\
\vdots \\
\left( 0, 0, 0, 0, 0, 1 \right), 
& \text{if } X[0] = \textit{Red } \text{and } X[1] = \textit{Male } \text{and } \frac{5}{6} \leq X[2] \leq 1,\\
\left( 1, 0, 0, 0, 0, 0 \right), 
& \text{if } X[0] = \textit{Blue } \text{and } X[1] = \textit{Female } \text{and } X[2] < \frac{1}{6}, \\
\vdots \\
\left( 0, 0, 0, 0, 0, 1 \right), 
& \text{if } X[0] = \textit{Blue } \text{and } X[1] = \textit{Female } \text{and } \frac{5}{6} \leq X[2] \leq 1.
\end{cases}   
\end{equation*}


\subsection{Parameter sensitivity}\label{A:Parameter sensitivity}

\begin{figure}[t]
    \centering
    \begin{minipage}[t]{\textwidth}
        \centering
        \begin{subfigure}[t]{0.24\textwidth}
            \centering
            \includegraphics[width=\linewidth]{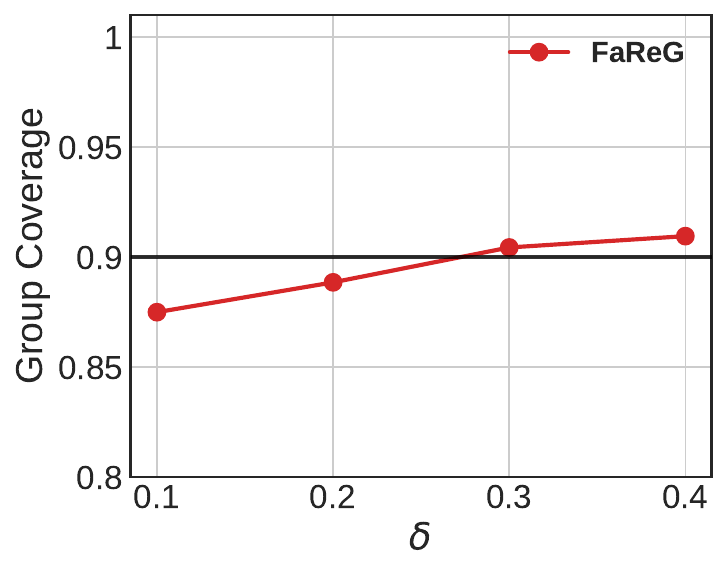}
            \caption{Group Coverage}\label{F:param-d-1}
        \end{subfigure}
        \hfill
        \begin{subfigure}[t]{0.24\textwidth}
            \centering
            \includegraphics[width=\linewidth]{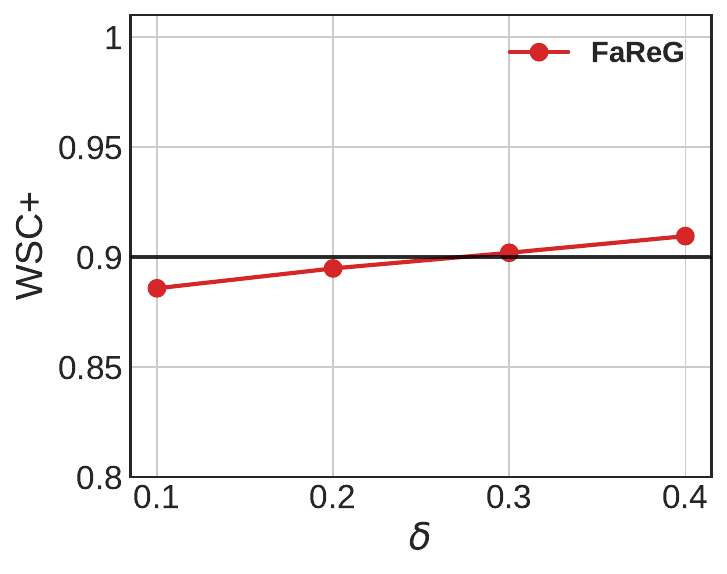}
            \caption{$\text{WSC}^+_n$}\label{F:param-d-2}
        \end{subfigure}
        \hfill
        \begin{subfigure}[t]{0.24\textwidth}
            \centering
            \includegraphics[width=\linewidth]{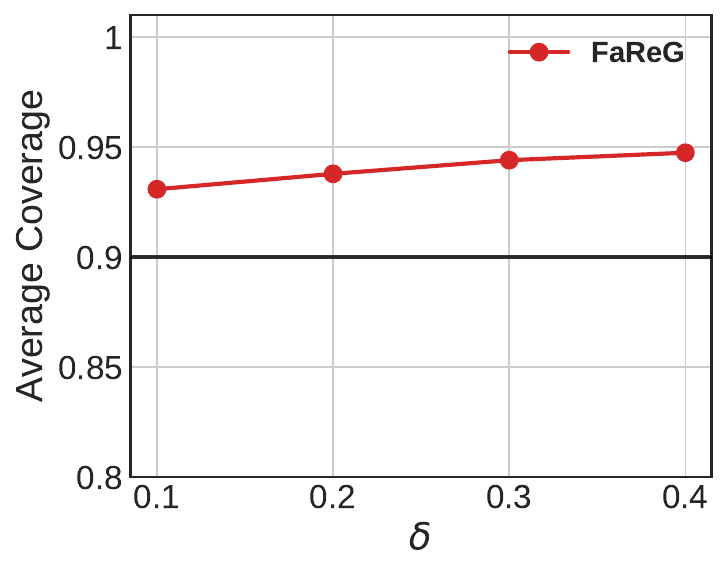}
            \caption{Average Coverage}\label{F:param-d-3}
        \end{subfigure}
        \hfill
        \begin{subfigure}[t]{0.24\textwidth}
            \centering
            \includegraphics[width=\linewidth]{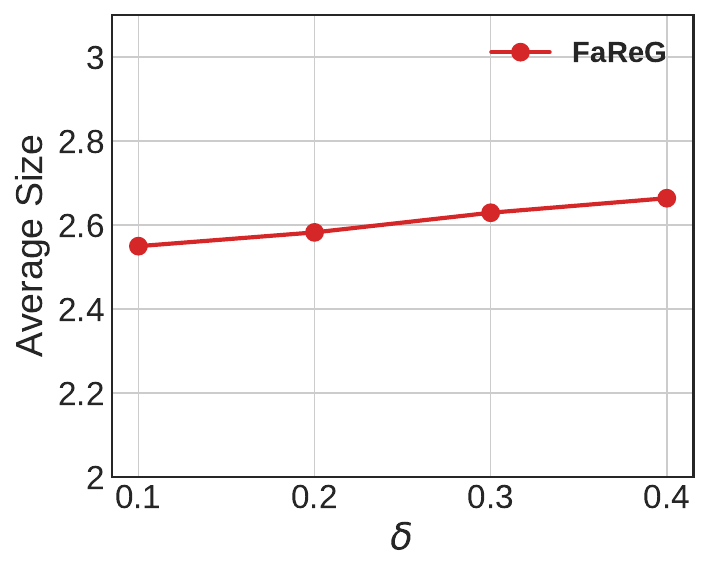}
            \caption{Average Size}\label{F:param-d-4}
        \end{subfigure}
        \vskip -0.5em
        \caption{Performance of prediction sets produced by our \name on synthetic data w.r.t. the selected group size proportion $\delta$. 
          \label{F:parameter-sensitive-1}
        }
    \end{minipage}
\end{figure}

\begin{figure}[t]
    \centering
    \begin{minipage}[t]{\textwidth}
        \centering
        \begin{subfigure}[t]{0.24\textwidth}
            \centering
            \includegraphics[width=\linewidth]{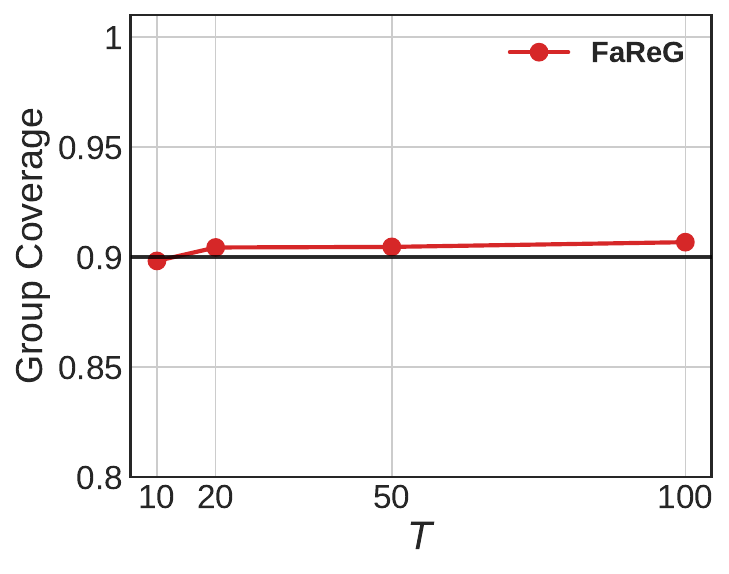}
            \caption{Group Coverage}\label{F:param-T-1}
        \end{subfigure}
        \hfill
        \begin{subfigure}[t]{0.24\textwidth}
            \centering
            \includegraphics[width=\linewidth]{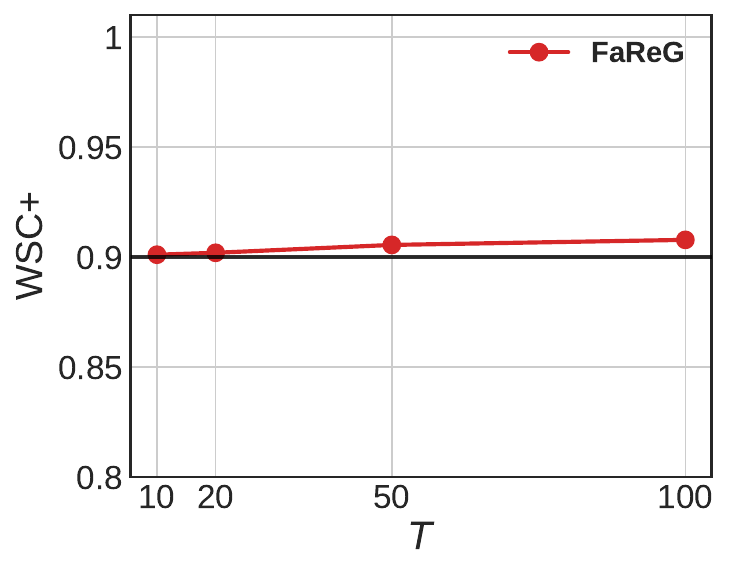}
            \caption{$\text{WSC}^+_n$}\label{F:param-T-2}
        \end{subfigure}
        \hfill
        \begin{subfigure}[t]{0.24\textwidth}
            \centering
            \includegraphics[width=\linewidth]{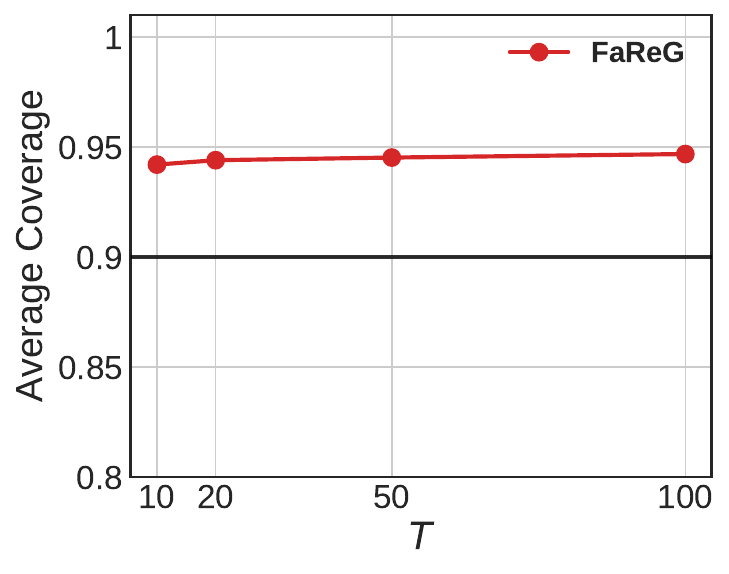}
            \caption{Average Coverage}\label{F:param-T-3}
        \end{subfigure}
        \hfill
        \begin{subfigure}[t]{0.24\textwidth}
            \centering
            \includegraphics[width=\linewidth]{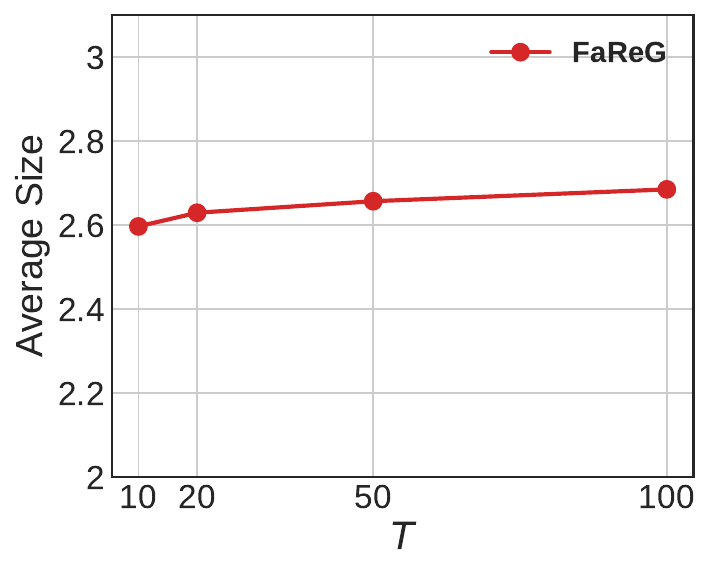}
            \caption{Average Size}\label{F:param-T-4}
        \end{subfigure}
        \vskip -0.5em
        \caption{Performance of prediction sets produced by our \name on synthetic data w.r.t. the sampling times $T$. 
          \label{F:parameter-sensitive-2}
        }
    \end{minipage}
\end{figure}

\begin{table*}[h]
    \centering
    \caption{Performance of prediction sets produced by our \name w.r.t the hyperparameter $\beta$.   
    \label{T:beta}}
    \resizebox{0.6\linewidth}{!}{
    \begin{sc}
    \begin{threeparttable}
      \begin{tabular}{lccccc}
       \toprule
        $\beta$ & 0.1 & 0.2 & 0.5 & 1.0 & 2.0  \\
        \midrule
        Group Coverage & 0.902 & 0.898 & 0.909 & 0.906 & 0.904 \\
        Average Size & 2.738 & 2.722 & 2.714 & 2.721 & 2.761 \\
      \bottomrule
      \end{tabular}
    \end{threeparttable}
    \end{sc}
 }
\end{table*}

In this section, we first investigate the sensitivity of two key parameters: the selected group size proportion $\delta$ and the number of sampling iterations $T$.
The results for $\delta$ and $T$ are presented in Fig.~\ref{F:parameter-sensitive-1} and Fig.~\ref{F:parameter-sensitive-2}, respectively.
Overall, the metrics Group Coverage and $\text{WSC}^+_n$ show relative insensitivity to the number of sampling iterations $T$, as illustrated in Fig.~\ref{F:param-T-1} and~\ref{F:param-T-2}.
In contrast, both Group Coverage and $\text{WSC}^+_n$ increase with the proportion $\delta$ of the selected group size relative to the entire dataset. This trend empirically provides implicit support for Proposition~\ref{prop:finite}.

We further conduct a sensitivity analysis of the hyperparameter $\beta$ (Eq.~\ref{Eq:Loss Funciton_2}) on synthetic data with a sample size of 2,000 and report the average results across 10 runs. 
The experimental results in Table~\ref{T:beta} show that, as $\beta$ varies, \name consistently achieves valid adaptive equalized coverage (0.9) with relatively steady efficiency.

\subsection{Group Visualization}\label{A:vision}

\begin{figure}[t]
\begin{minipage}{1.0\linewidth}
\centering
    \begin{subfigure}[t]{0.32\textwidth}
        \centering
        \includegraphics[width=\linewidth]{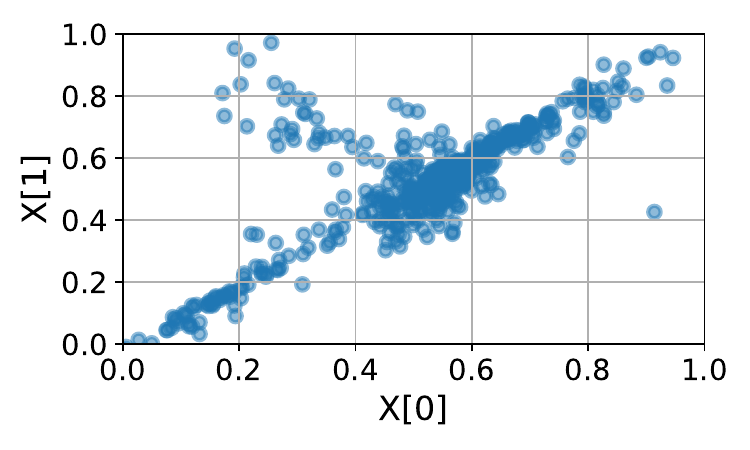}
        \caption{$\vert \hat{\mathcal{G}}_{\mathbf{s}}\vert / N=0.31$ }\label{F:vision-1}
    \end{subfigure}
    \begin{subfigure}[t]{0.32\textwidth}
        \centering
        \includegraphics[width=\linewidth]{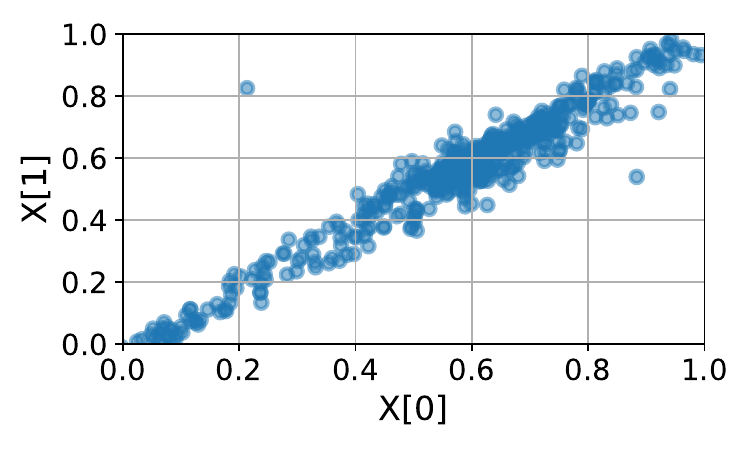}
        \caption{$\vert \hat{\mathcal{G}}_{\mathbf{s}}\vert / N=0.39$}\label{F:vision-2}
    \end{subfigure}
    \begin{subfigure}[t]{0.32\textwidth}
        \centering
        \includegraphics[width=\linewidth]{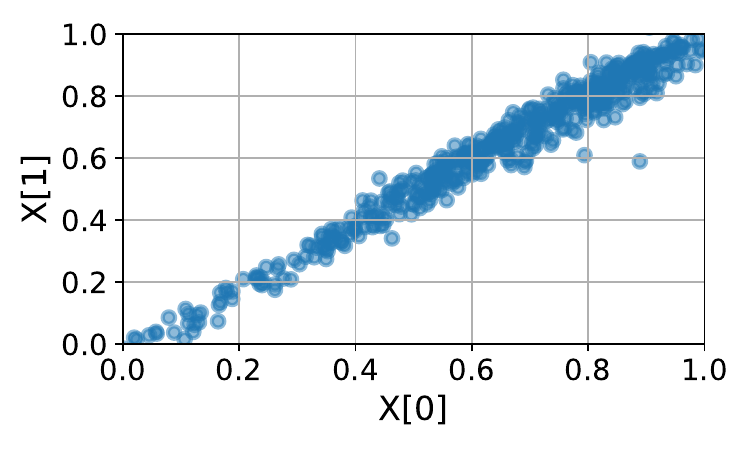}
        \caption{$\vert \hat{\mathcal{G}}_{\mathbf{s}}\vert / N=0.47$}\label{F:vision-3}
    \end{subfigure}
    \caption{The visualization results of reconstruction $\hat{X}$ when $\vert \hat{\mathcal{G}}_{\mathbf{s}}\vert / N$ increases.
    The latent representation $z$ generally captures an XNOR relationship between $X[0]$ and $X[1]$.}\label{F:vision}
\end{minipage}
\vspace{-0.8em}
\end{figure}

To analyze how features $X[0]$ and $X[1]$ contribute to group membership in $\hat{\mathcal{G}}_{\mathbf{s}}$, we perturb the encoding $z$ and examine the resulting reconstructions $\hat{X}$. Taking the sample size of 4,000 on the synthetic dataset as an example, we randomly select one run from 10 repeated trials and add perturbations of $+0.003$ and $+0.006$ to the fourth dimension of $z$, respectively. The reconstructed features $\hat{X}$ are visualized in Fig.~\ref{F:vision}.

Fig.~\ref{F:vision-1} shows that, without perturbation, the latent representation $z$ generally captures an XNOR relationship between $X[0]$ and $X[1]$, indicating that the encoder effectively filters out irrelevant feature information. After applying perturbations (see Fig.~\ref{F:vision-2} and~\ref{F:vision-3}), the XNOR pattern becomes more pronounced as ${\vert \hat{\mathcal{G}}_{\mathbf{s}} \vert / N}$ increases, revealing a positive correlation between $X[0] \odot X[1]$ and membership in $\hat{\mathcal{G}}_{\mathbf{s}}$. This result strengthens the interpretability of our approach by demonstrating that the representation-based groups reflect meaningful feature interactions.





\subsection{Other synthetic setups beyond XNOR}
In this part, we construct a more complicated setup by partitioning the synthetic data in Section~\ref{sec:synthetic data} into eight diverse subgroups: Male Child, Male Youth, Male Middle-Aged, Male Elderly, Female Child, Female Youth, Female Middle-Aged, and Female Elderly. 
We then impose algorithmic biases, using a procedure analogous to that in Section~\ref{sec:synthetic data}, on four of these subgroups: Male Child, Male Youth, Female Middle-Aged, and Female Elderly. 
In experiments with a sample size of 2,000, we compare our method, \name, with the SOTA method AFCP2. 
The experimental results in Table~\ref{T:new synthetic setup}, averaged over 10 runs, also demonstrate that \name achieves the ideal conditional coverage (0.9) while sacrificing very little efficiency.

\begin{table*}[h]
    \centering
    \caption{Performance of prediction sets produced by different CP methods on the other synthetic setups beyond XNOR.   
    \label{T:new synthetic setup}}
    \resizebox{0.83\linewidth}{!}{
    \begin{sc}
    \begin{threeparttable}
      \begin{tabular}{lcccc}
       \toprule
        {Method} & Group Coverage & $\text{WSC}^+_n$ & Average Coverage & Average Size   \\
        \midrule
        AFCP2 & 0.849 & 0.866 & 0.926 & 2.68 \\
        \name & 0.901 & 0.902 & 0.942 & 2.74 \\
      \bottomrule
      \end{tabular}
    \end{threeparttable}
    \end{sc}
 }
\end{table*}


\subsection{CP methods without adaptive group selections}

We next build \name upon the results of CondCP~\citep{gibbs2025conformal}, which does not consider the automatic selection of the sensitive groups as discussed in Section~\ref{sec:rel}.
As shown in Table~\ref{T:condCP}, experiments on synthetic data with a sample size of 2,000 reveal that CondCP may lack the expressiveness needed to capture complex unfairly treated groups. 
Consequently, our \name method further improves conditional coverage based on CondCP. 

\begin{table*}[h]
    \centering
    \caption{Performance of prediction sets produced by different CP methods with and without adaptive group selections.   
    \label{T:condCP}}
    \resizebox{0.9\linewidth}{!}{
    \begin{sc}
    \begin{threeparttable}
      \begin{tabular}{lcccc}
       \toprule
        {Method} & Group Coverage & $\text{WSC}^+_n$ & Average Coverage & Average Size   \\
        \midrule
        CondCP & 0.819 & 0.815 & 0.899 & 2.53 \\
        \name & 0.904 & 0.899 & 0.944 & 2.76 \\
        \midrule
        CondCP + \name & 0.903 & 0.897 & 0.940 & 2.72 \\
      \bottomrule
      \end{tabular}
    \end{threeparttable}
    \end{sc}
 }
\end{table*}


\subsection{More real-world datasets}

In this part, we evaluate our method using the standard ACSIncome dataset from Folktables~\citep{ding2021retiring}, which comprises nine features: Age, Class of worker, Educational attainment, Marital status, Occupation, Place of birth, Relationship, Working hours, Gender, and Race. 
Among these, Age, Marital status, Sex, and Race are designated as sensitive features. The target variable, i.e., a person’s income, is evenly divided into four brackets, forming a 4-class classification task. We further partition the Age feature into four intervals: [17–26, 27–36, 37–46, 47–56]. A specific subgroup is defined by the condition (Age = 17–26 \& Sex = Male) or (Age = 47–56 \& Sex = Female), into which we manually inject algorithmic biases, following a procedure similar to that in Section~\ref{sec:nursery data}. 
For time and cost constraints, particularly given the high complexity of the baseline AFCP method, we randomly sample 20,000 instances from the combined datasets of CA, NY, TX, AK, and AL. The proposed \name is then compared against the state-of-the-art approach, AFCP2.
Experimental results in Table~\ref{T:folktables} (averaged over 10 independent runs) show that \name still consistently achieves the valid conditional coverage of 0.9 while maintaining competitive efficiency.

\begin{table*}[h]
    \centering
    \caption{Performance of prediction sets produced by different CP methods on the ACSIncome data.   
    \label{T:folktables}}
    \vskip -1.0em
    \resizebox{0.83\linewidth}{!}{
    \begin{sc}
    \begin{threeparttable}
      \begin{tabular}{lcccc}
       \toprule
        {Method} & Group Coverage & $\text{WSC}^+_n$ & Average Coverage & Average Size   \\
        \midrule
        AFCP2 & 0.839 & 0.883 & 0.918 & 3.17 \\
        \name & 0.898 & 0.907 & 0.936 & 3.31 \\
      \bottomrule
      \end{tabular}
    \end{threeparttable}
    \end{sc}
 }
\end{table*}

\end{document}